\pdfoutput=1
\documentclass{article} 
\usepackage{iclr2018_conference,times}
\iclrfinalcopy

\usepackage{mathtools}
\usepackage{amsmath,amsfonts,amssymb}       
\usepackage{amsthm}
\usepackage{bbm}
\usepackage{bm}
\usepackage[all]{xy}
\usepackage{graphicx}
\usepackage{subcaption}
\usepackage[font=small]{caption}

\usepackage[utf8]{inputenc} 
\usepackage[T1]{fontenc}    
\usepackage{hyperref}       
\usepackage{url}            
\usepackage{booktabs}       
\usepackage{amsfonts}       
\usepackage{nicefrac}       
\usepackage{microtype}      
\usepackage{float}
\usepackage{algorithm}      
\usepackage{algpseudocode}  
\usepackage{algorithmicx}   

\newcommand{\LL}{\mathcal{L}}

\newcommand\EE{\mathbb{E}}

\newcommand\PP{\mathbb{P}}

\newcommand{\norm}[1]{\left\lVert#1\right\rVert} 

\newcommand{\M}[2]{T_{#1\mid#2}}
\newcommand{\Minv}[2]{T^{\dagger}_{#1\mid#2}}

\newcommand{\KL}[2]{D_{KL}(#1\mid\mid#2)}
\newcommand{\JS}[2]{D_{JS}(#1\mid\mid#2)}
\newcommand{\XX}[2]{D_{\chi^2}(#1\mid\mid#2)}
\newcommand{\FD}[2]{D_{f}(#1\mid\mid#2)}

\newtheorem{proposition}{Proposition}
\newtheorem{lemma}{Lemma}
\theoremstyle{definition}

\theoremstyle{definition}

\title{Hierarchical Adversarially Learned \\
      \centering Inference}

\def\equationautorefname~#1\null{%
  Eq.~#1\null
}
\def\figureautorefname~#1\null{%
  Fig.~#1\null
}
\author{Mohamed Ishmael Belghazi$^{1}$, Sai Rajeswar$^{1}$, Olivier Mastropietro$^{1}$,\\
  \bf{Negar Rostamzadeh$^{2}$}, \bf{Jovana Mitrovic$^{2}$} and \bf{Aaron Courville$^{1\dagger}$}\\
  $^1$ MILA, Universit\'e de Montr\'eal, \texttt{firstname.lastname@umontreal.ca}.\\
  $^2$ Element AI,
  \texttt{negar@elementai.com}.\\
  $^3$ DeepMind, \texttt{mitrovic@stats.ox.ac.uk}. \\
  $^\dagger$CIFAR Fellow.\\
}

\begin{document}

\maketitle

\begin{abstract}
We propose a novel hierarchical generative model with a simple
  Markovian structure and a corresponding inference model. Both the
  generative and inference model are trained using the adversarial
  learning paradigm. We demonstrate that the hierarchical structure
  supports the learning of progressively more abstract
  representations as well as providing semantically meaningful reconstructions with different levels of fidelity. 
  Furthermore, we show that minimizing the Jensen-Shanon divergence between the generative and inference network is enough to minimize the reconstruction error. 
 The resulting semantically meaningful hierarchical latent structure discovery is exemplified on the CelebA dataset. 
 There, we show that the features learned by our model in an unsupervised way outperform the best handcrafted features. Furthermore, the extracted features remain competitive when compared to several recent deep supervised approaches on an attribute prediction task on CelebA. 
  Finally, we leverage the model's inference network to achieve state-of-the-art performance on a semi-supervised variant of the MNIST digit classification task.  
 \end{abstract}

\section{Introduction}
\label{intro}

Deep generative models represent powerful approaches to modeling highly
complex high-dimensional data. There has been a lot of recent
research geared towards the advancement of deep generative
modeling strategies, including Variational
Autoencoders (VAE) ~\citep{kingma2013auto}, autoregressive
models~\citep{van2016conditional, oord2016pixel} and hybrid
models~\citep{GulrajaniKATVVC16,NguyenYBDC16}. However, Generative Adversarial Networks (GANs) ~\citep{goodfellow2014generative} have emerged as the learning paradigm of choice across a varied range of tasks, especially in computer vision~\cite{ZhuPIE17}, simulation and robotics \cite{FinnGL16} \cite{ShrivastavaPTSW16}. GANs cast the learning of a generative network in the form of a game between the
generative and discriminator networks. While the discriminator is
trained to distinguish between the true and generated examples, the generative model is trained to fool the discriminator. Using a discriminator network in GANs avoids the need for an explicit reconstruction-based loss function. This allows this model class to generate visually sharper images
than VAEs while simultaneously enjoying faster sampling than autoregressive models.

Recent work, known as either ALI~\citep{dumoulin2016adversarially} or BiGAN~\citep{donahue2016adversarial}, has shown that the adversarial learning paradigm can be extended to incorporate the learning of an inference network. While the inference network, or encoder, maps training examples $\bm{x}$ to a latent space variable
$\bm{z}$, the decoder plays the role of the standard GAN generator mapping from space of the latent variables (that is typically sampled from some factorial
distribution) into the data space. In ALI, the discriminator is trained to distinguish between the encoder and the decoder, while the encoder and decoder are trained to conspire together to fool the discriminator. Unlike some approaches that
hybridize VAE-style inference with GAN-style generative learning
(e.g. \cite{larsen2015autoencoding}, \cite{ChenDHSSA16}),
the encoder and decoder in ALI use a purely adversarial approach. One big advantage of adopting  an adversarial-only formalism is demonstrated by the
high-quality of the generated samples. Additionally, we are given a
mechanism to infer the latent code associated with a true data example.

One interesting feature highlighted in the original ALI
work~\citep{dumoulin2016adversarially} is that even though the encoder and
decoder models are never explicitly trained to perform reconstruction, this can nevertheless be easily done by projecting data samples via the encoder into the latent space, copying these values across to the latent variable layer of the decoder and projecting them back to the data space. 
Doing this yields reconstructions that often preserve some semantic features of the
original input data, but are perceptually relatively different from the original
samples. These observations naturally lead to the question of the source of the
discrepancy between the data samples and their ALI reconstructions. Is the
discrepancy due to a failure of the adversarial training paradigm, or is it
due to the more standard challenge of compressing the information from the data into a rather restrictive latent feature vector?
\cite{Ulyanov2017} show that an improvement in 
reconstructions is achievable when additional terms which explicitly minimize reconstruction error in the data space are added to the training objective. \cite{Li2017_ALICE} palliates to the non-identifiability issues pertaining to bidirectional adversarial training by augmenting the generator's loss with an adversarial cycle consistency loss.

In this paper we explore issues surrounding the representation of complex,
richly-structured data, such as natural images, in the context of a novel, hierarchical generative model, Hierarchical Adversarially Learned Inference (HALI), which represents a hierarchical extension of ALI. We show that within a purely adversarial training paradigm, and by exploiting the model's hierarchical structure, we can modulate the perceptual fidelity of the
reconstructions. 
We provide theoretical arguments for why HALI's
adversarial game should be sufficient to minimize the reconstruction cost and show empirical evidence supporting this perspective. Finally, we evaluate the usefulness of the learned representations on a semi-supervised task on MNIST and an attribution prediction task on the CelebA dataset. 

\section{Related work}
\label{related work}
Our work fits into the general trend of hybrid approaches to
generative modeling that combine aspects of VAEs and GANs. For example, Adversarial Autoencoders \citep{makhzani2015adversarial} replace the Kullback-Leibler divergence that appears in the training objective for VAEs with an adversarial discriminator that learns to distinguish between samples from the approximate posterior and the prior. A second line of research has been directed towards replacing the reconstruction penalty from the VAE objective with GANs or other kinds of auxiliary losses. Examples of this include \cite{larsen2015autoencoding} that combines the GAN generator and the VAE decoder into one network
and \cite{lamb2016discriminative} that uses the loss of a  pre-trained classifier as an additional reconstruction loss in the VAE objective. Another research direction has been focused on augmenting GANs with inference machinery. One particular approach is given by \cite{dumoulin2016adversarially, donahue2016adversarial}, where, like in our approach, there is a separate inference network that is jointly trained with the usual GAN discriminator and generator. \cite{karaletsos2016adversarial} presents a theoretical framework to jointly train inference networks and generators defined on directed acyclic graphs by leverage multiple discriminators defined nodes and their parents.
Another related work is that of \cite{HuangLPHB16} which takes advantage
of the representational information coming from a pre-trained discriminator. 
Their model decomposes the data generating task into multiple subtasks, where each level outputs an intermediate representation conditioned on the representations  from higher level. A stack of discriminators is employed to provide signals for
these intermediate representations. The idea of stacking discriminator can be traced back to \cite{denton2015deep} which used used a succession of convolutional networks within a Laplacian pyramid framework to progressively increase the resolution of the generated images.

\section{Hierachical Adversarially Learned Inference}
The goal of generative modeling is to capture the data-generating process with a probabilistic model. Most real-world data is highly complex and thus, the exact modeling of the underlying probability density function is usually computationally intractable. Motivated by this fact, GANs \citep{goodfellow2014generative} model the data-generating distribution as a transformation of some fixed distribution over latent variables. In particular, the adversarial loss, through a discriminator network, forces the generator network to produce samples that are close to those of the data-generating distribution. 
%
While GANs are flexible and provide good approximations to the true data-generating mechanism, their original formulation does not permit inference on the latent variables.
In order to mitigate this, Adversarially Learned Inference (ALI) \citep{dumoulin2016adversarially} extends the GAN framework to include an inference network that encodes the data into the latent space.
The discriminator is then trained to discriminate between the joint distribution of the data and latent causes coming from the generator and inference network. 
Thus, the ALI objective encourages a matching of the two joint distributions, which also results in all the marginals and conditional distributions being matched. This enables inference on the latent variables.

We endeavor to improve on ALI in two aspects. First, as reconstructions from ALI only
loosely match the input on a perceptual level, we want to achieve better perceptual matching in the reconstructions. 
Second, we wish to be able to compress the observables, $\bm{x}$, using a sequence of composed features maps, leading to a distilled hierarchy of stochastic latent representations, denoted by $\bm{z}_1$ to $\bm{z}_L$. 
Note that, as a consequence of the data processing inequality\citep{cover2012elements}, latent representations higher up in the hierarchy cannot contain more information than those situated lower in the hierarchy. In information-theoretic terms, the conditional entropy of the observables given a latent variable is non-increasing as we ascend the hierarchy. 
This loss of information can be seen as responsible for the perceptual discrepancy observed in ALI's reconstructions. 
Thus, the question we seek to answer becomes:
  How can we achieve high perceptual fidelity of the data reconstructions while also having a compressed latent space that is strongly coupled with the observables?
In this paper, we propose to answer this using a novel model, Hierarchical Adversarially Learned Inference (HALI), that uses a simple hierarchical Markovian inference network that is matched through adversarial training to a similarly constructed generator network. Furthermore, we discuss the hierarchy of reconstructions induced by the HALI's hierarchical inference network and show that the resulting reconstruction errors are implicitly minimized during adversarial training. Also, we leverage HALI's hierarchial inference network to offer a novel approach to semi-supervised learning in generative adversarial models.
  
\subsection{A Model for Hierarchical Features}
Denote by $\mathcal{P}(S)$ the set of all probability measures
on some set $S$. Let $\M{Z}{X}$ be a Markov kernel associating to each element $x\in X$ a probability measure $\PP_{Z \mid X=x} \in \mathcal{P}(Z)$.
Given two Markov kernels $\M{W}{V}$ and $\M{V}{U}$, a further Markov
kernel can be defined by composing these two and then marginalizing over $V$, i.e. $\M{W}{V} \circ \M{V}{U}: U \to
\mathcal{P}(W)$.
Consider a set of random variables $\bm{x}, \bm{z}_1, \dots, \bm{z}_L$. Using the composition operation, we can construct a hierarchy of Markov kernels or \emph{feature transitions} as
\begin{equation}
  \label{eq:encoder_fmap}
  \M{\bm{z}_L}{\bm{x}} = \M{\bm{z}_{L}}{\bm{z}_{L - 1}} \circ \dots \circ\M{\bm{z}_{1}}{\bm{x}}.
\end{equation}
A desirable property for these feature transitions is to have some form of inverses. Motivated by this, we define the adjoint feature transition as $\Minv{\bm{z}_l}{\bm{z}_{l - 1}} = \M{\bm{z}_{l - 1}}{\bm{z}_l}$. From this, we see that
\begin{equation}
  \label{eq:decoder_fmap}
  \Minv{\bm{z}_L}{\bm{x}} = \M{\bm{x}}{\bm{z}_L} = \M{\bm{z}_{1}}{\bm{z}_{2}} \circ \dots \circ\M{\bm{z}_{L - 1}}{\bm{z}_{L}}.
\end{equation}
This can be interpreted as the generative mechanism of the latent variables given the data being the "inverse" of the data-generating mechanism given the latent variables. 
Let  $q(\bm{x})$ denote the distribution of the data and $p(\bm{z}_L)$ be the
prior on the latent variables. Typically the prior will be a simple distribution, e.g. a
standard Gaussian $p(\bm{z}_L) = \mathcal{N}(\bm{0} \mid \bm{I})$ .


The composition of Markov kernels in \autoref{eq:encoder_fmap}, mapping data samples
$\bm{x}$ to samples of the latent variables $\bm{z}_L$ using $\bm{z}_1, \dots, \bm{z}_{L - 1}$ constitutes the encoder. Similarly, the composition of kernels
in \autoref{eq:decoder_fmap} mapping prior samples of $\bm{z}_L$ to data samples $\bm{x}$
through $\bm{z}_{L - 1}, \dots, \bm{z}_1$ constitutes the decoder.
Thus, the joint distribution of the encoder can be written as 

\begin{equation}
  \label{eq:encoder_joint}
  q(\bm{x}, \dots, \bm{z}_L) = \prod_{l=2}^{L} q(\bm{z}_l \mid \bm{z}_{l -1}) \, q(\bm{z}_1 \mid \bm{x}) \, q(\bm{x}),
\end{equation}
\label{def:encoder}
while the joint distribution of the decoder is given by

\begin{equation}
  \label{eq:decoder_joint}
  p(\bm{x}, \dots, \bm{z}_L) =  p(\bm{x} \mid \bm{z}_1) \, \prod_{l=2}^{L} \, p(\bm{z}_{l - 1} \mid \bm{z}_{l}) \, p(\bm{z}_L).
\end{equation}
\label{def:decoder}

The encoder and decoder distributions can be visualized graphically as
$$
\xymatrix{
\bm{x} \ar@/^/[rr]^{\M{\bm{z}_1}{\bm{x}}} && \ar@/^/[ll]^{\M{\bm{x}}{\bm{z}_1}} \bm{z}_1 \ar@/^/[rr]^{\M{\bm{z}_2}{\bm{z}_1}} && \ar@/^/[ll]^{\M{\bm{z}_1}{\bm{z}_2}} \bm{z}_2 \ar@/^/[rr]^{\M{\bm{z}_3}{\bm{z}_2}} && \ar@/^/[ll]^{\M{\bm{z}_2}{\bm{z}_3}} \dots \ar@/^/[rr]^{\M{\bm{z}_L}{\bm{z}_{L-1}}} && \ar@/^/[ll]^{\M{\bm{z}_{L-1}}{\bm{z}_L}} \bm{z}_L
}
$$
Having constructed the joint distributions of the encoder and decoder, we can now match these distributions through adversarial
training. It can be shown that, under an
ideal (non-parametric) discriminator, this is equivalent to minimizing the Jensen-Shanon divergence between the joint \autoref{eq:encoder_joint} and
\autoref{eq:decoder_joint}, see ~\citep{dumoulin2016adversarially}. Algorithm~\ref{alg:hali} details the training procedure.

\begin{algorithm}[t]
\begin{algorithmic}
    \State $\theta_{g}, \theta_{d} \gets \text{initialize network parameters}$
    \Repeat
        \For{$m \in \{1, \dots, M\}$}
		\State $\bm{\hat{z}}_0^{(m)} \sim q(\bm{x})$
            \Comment{Sample from the dataset}
		\State $\bm{z}_L^{(m)} \sim p(\bm{z})$
            \Comment{Sample from the prior}
        \For{$l \in \{1, \dots, L\}$}
        	\State $\bm{\hat{z}}^{m}_l \sim q(\bm{z}_l \mid \bm{\hat{z}}_{l - 1})$
            \Comment{Sample from each level in the encoder's hierarchy}
        \EndFor
        \For{$l \in \{L \dots 1\}$}
        	\State $\bm{z}^{(m)}_{l - 1} \sim p(\bm{z}_{l - 1} \mid \bm{z}_{l})$
            \Comment{Sample from each level in the decoder's hierarchy}
        \EndFor
            \State $\rho_q^{(m)} \gets D(\bm{\hat{z}}_0^{(m)}, \dots, \bm{\hat{z}}^{(m)}_l)$ 
            \Comment{Get discriminator predictions on encoder's distribution}
            
            \State $\rho_p^{(m)} \gets D(\bm{z}_0^{(m)}, \dots, \bm{z}^{(m)}_l)$ 
            \Comment{Get discriminator predictions on decoder's distribution}
        \EndFor
        \State $\mathcal{L}_d \gets
            -\frac{1}{M} \sum_{i=1}^M \log(\rho_q^{(i)})
            -\frac{1}{M} \sum_{i=1}^M\ log(1 - \rho_p^{(i)})$
            \Comment{Compute discriminator loss}
        \State $\mathcal{L}_g \gets
            -\frac{1}{M} \sum_{i=1}^M \log(1 - \rho_q^{(i)})
            -\frac{1}{M} \sum_{i=1}^M \log(\rho_p^{(i)})$
            \Comment{Compute generator loss}
        \State $\theta_d \gets \theta_d - \nabla_{\theta_d} \mathcal{L}_d$
            \Comment{Gradient update on discriminator network}
        \State $\theta_g \gets \theta_g - \nabla_{\theta_g} \mathcal{L}_g$
            \Comment{Gradient update on generator networks}
    \Until{convergence}
\end{algorithmic}
\caption{\label{alg:hali} HALI training procedure.}
\end{algorithm}
\subsection{A hierarchy of reconstructions} \label{sec:high_recon}
The Markovian character of both the encoder and decoder implies a hierarchy of reconstructions in the decoder. In particular, for a given observation $\bm{x} \sim
p(\bm{x})$, the model yields $L$ different reconstructions
$\hat{\bm{x}}_{l} \sim \M{\bm{x}}{\bm{z}_l} \circ \M{\bm{z}_l}{\bm{x}} $ for
  $l \in \{1, \dots, L\}$ with $\hat{\bm{x}}_l$ the reconstruction of the $\bm{x}$ at the $l$-th level of the hierarchy.
Here, we can think of $\M{\bm{z}_{l}}{\bm{x}}$ as projecting $\bm{x}$ to the $l$-th
intermediate representation and $\M{\bm{x}}{\bm{z}_{l}}$ as projecting it back to the input
space. 
Then, the \emph{reconstruction error} for a given input
$\bm{x}$ at the $l$-th hierarchical level is given by
\begin{equation}
\LL^{l}(\bm{x}) = \EE_{\bm{z}_l \sim
  \M{\bm{z}_l}{\bm{x}}}[-\log(p(\bm{x} \mid \bm{z}_l))].
  \end{equation}
Contrary to models that try to merge autoencoders and adversarial
models, e.g. ~\cite{rosca2017variational, larsen2015autoencoding}, HALI
does not require any additional terms in its loss function in order to minimize the above reconstruction error. 
Indeed, the reconstruction errors at the different levels of HALI are minimized down to the amount of information about $\bm{x}$ that a given level of the hierarchy is able to encode as training proceeds. 
Furthermore, under an optimal discriminator, training in HALI minimizes the Jensen-Shanon divergence between 
  $q(\bm{x}, \bm{z}_1, \dots, \bm{z}_L)$ and $p(\bm{x}, \bm{z}_1, \dots,
  \bm{z}_L)$ as formalized in Proposition~\ref{prop:js_to_rec} below. Furthermore, the interaction between the reconstruction error and training dynamics is captured in Proposition~\ref{prop:js_to_rec}.
  \begin{proposition}
    \label{prop:js_to_rec}
    Assuming $q(\bm{x}, \bm{z}_l) $ is bounded away for zero for all $l \in \{1, \dots, L\}$, we have that
      \begin{equation}
        \EE_{\bm{x} \sim q(\bm{x})}[\mathcal{L}^l(\bm{x})] - H(\bm{x} \mid \bm{z}_l) \leq K\, \JS{p(\bm{x}, \bm{z}_1, \dots, \bm{z}_L)}{q (\bm{x}, \bm{z}_1, \dots, \bm{z}_L)},
      \end{equation}
      where $H(\bm{x} \mid \bm{z}_l)$ is computed under the encoder's distribution and $K$ is as defined in Lemma 2 in the appendix.
    \end{proposition}

On the other hand, proposition ~\ref{prop:rec_to_mi_h} below relates the intermediate representations in the hierarchy to the corresponding induced reconstruction error.
   \begin{proposition}
     \label{prop:rec_to_mi_h}
     For any given latent variable $\bm{z}_{l}$,
     \begin{equation}
     \EE_{\bm{x} \sim q_{\bm{x}}}[\EE_{\bm{z} \sim \M{\bm{z}_l}{\bm{x}}}[-\log p(\bm{x}\mid \bm{z}_{l})]] \geq H(\bm{x} \mid \bm{z}_l)
     \end{equation}
     i.e. the reconstruction error is an upper bound on $H(\bm{x} \mid \bm{z}_l)$.
   \end{proposition}

In summary, Propositions~\ref{prop:js_to_rec} and ~\ref{prop:rec_to_mi_h} establish the
dynamics between the hierarchical representation learned by the inference network, the reconstruction errors and 
the adversarial matching of the joint distributions \autoref{eq:encoder_joint} and \autoref{eq:decoder_joint}.  The proofs on the two propositions above are deferred to the appendix.
Having  theoretically established the interplay between 
layer-wise reconstructions and the training mechanics, 
we now move to the empirical evaluation of HALI.
   
\section{Empirical Analysis: Setup}
We designed our experiments with the objective of addressing the following questions: Is HALI successful in improving the fidelity perceptual reconstructions?  Does HALI induces a semantically meaningful representation of the observed data? Are the learned representations useful for downstream classification tasks?  
All of these questions are considered in turn in the following sections.
     
We evaluated HALI on four datasets, CIFAR10 \citep{krizhevsky2009learning}, SVHN
\citep{netzer2011reading}, ImageNet 128x128
\citep{russakovsky2015imagenet} and CelebA \citep{liu2015deep}.
We used two conditional hierarchies in all experiments 
with the Markov kernels parametrized by conditional isotropic
Gaussians. For SVHN, CIFAR10 and CelebA the resolutions of two level latent variables
are $\bm{z}_1 \in \mathbb{R}^{64\times16\times 16}$ and $\bm{z}_2 \in \mathbb{R}^{256}$. 
For ImageNet, the resolutions is $\bm{z}_1 \in \mathbb{R}^{64\times32\times 32}$ and $\bm{z}_2 \in \mathbb{R}^{256}$.

For both the encoder and decoder, we use residual blocks\citep{HeZRS15} with skip connections between the blocks in conjunction with batch normalization\citep{IoffeS15}.
We use convolution with stride 2 for downsampling in the encoder and bilinear upsampling in the decoder. In the discriminator, we use consecutive
stride 1 and stride 2 convolutions and weight normalization \citep{salimans2016weight}. To
regularize the discriminator, we apply dropout every 3 layers with a
probability of retention of 0.2. We also add Gaussian noise with
standard deviation of 0.2 at the inputs of the discriminator and the encoder.

\section{Empirical Analysis I: Reconstructions}\label{subsec:exp_rec}

One of the desired objectives of a generative model is to reconstruct the input 
images from the latent representation. We show that HALI offers improved perceptual reconstructions relative to the (non-hierarchical) ALI model.

\label{exp:recons}
\begin{figure}
\begin{minipage}{\textwidth}
  \begin{subfigure}[t]{0.5\textwidth}
    \centering
    \includegraphics[width=6.8cm]{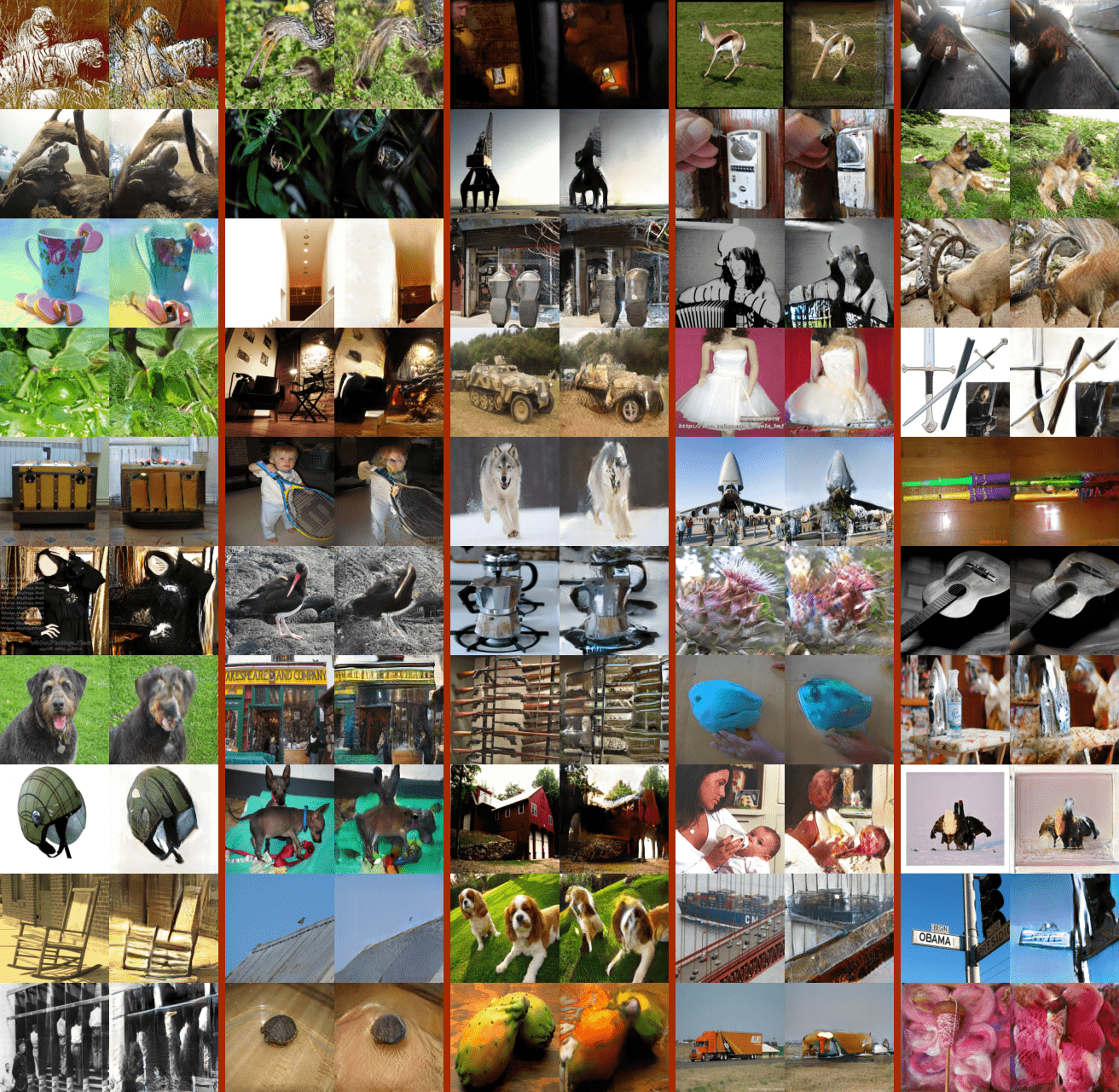}
    \caption{ImageNet128 from $\bm{z}_1$}
   \end{subfigure}
   \begin{subfigure}[t]{0.5\textwidth}
    \centering
    \includegraphics[width=6.8
    cm]{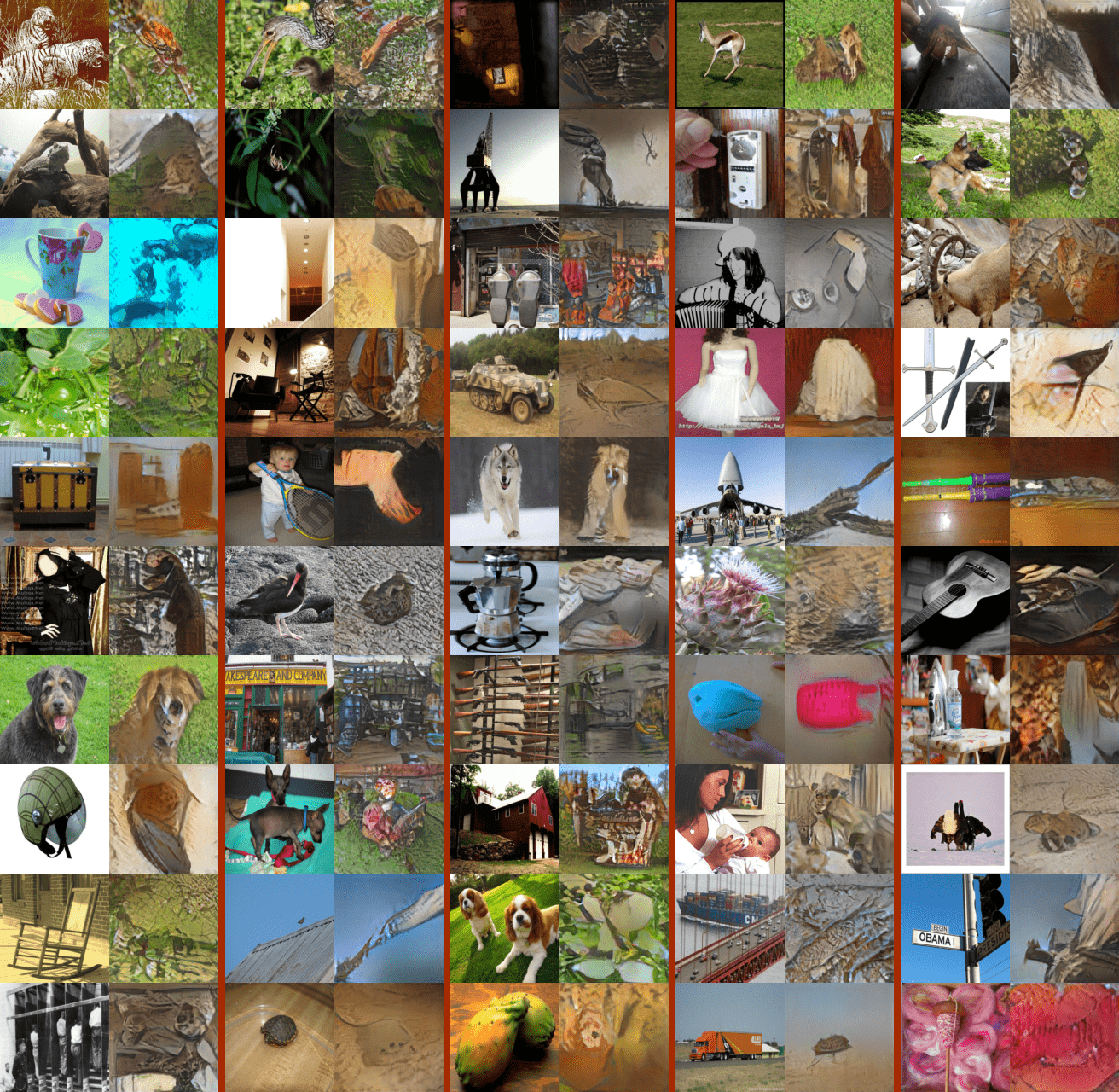}
    \caption{ImageNet128 from $\bm{z}_2$}
  \end{subfigure}
\end{minipage}%
  \caption{ImageNet128 reconstructions from $\bm{z}_1$ and $\bm{z}_2$. 
    Odd columns corresponds to examples from the validation set while even columns are 
    the model's reconstructions} \label{fig:recons}
\end{figure}

\subsection{Qualitative analysis}
First, we present reconstructions obtained on ImageNet. Reconstructions from SVHN and CIFAR10 can be seen in \autoref{fig:recons_appendix} in the appendix.
\autoref{fig:recons} highlights HALI's ability to reconstruct the input samples with high fidelity. We observe that reconstructions from the first level 
of the hierarchy exhibit local differences in the natural images, while reconstructions 
from the second level of the hierarchy displays global change. Higher conditional 
reconstructions are more often than not reconstructed as a different member of the same class.  
Moreover, we show in \autoref{fig:samples} that this increase in reconstruction fidelity 
does not impact the quality of the generative samples from HALI's decoder.

\label{exp:samples}
\begin{figure}
  \begin{subfigure}[t]{0.5\textwidth}
    \centering
    \includegraphics[width=6.8cm]{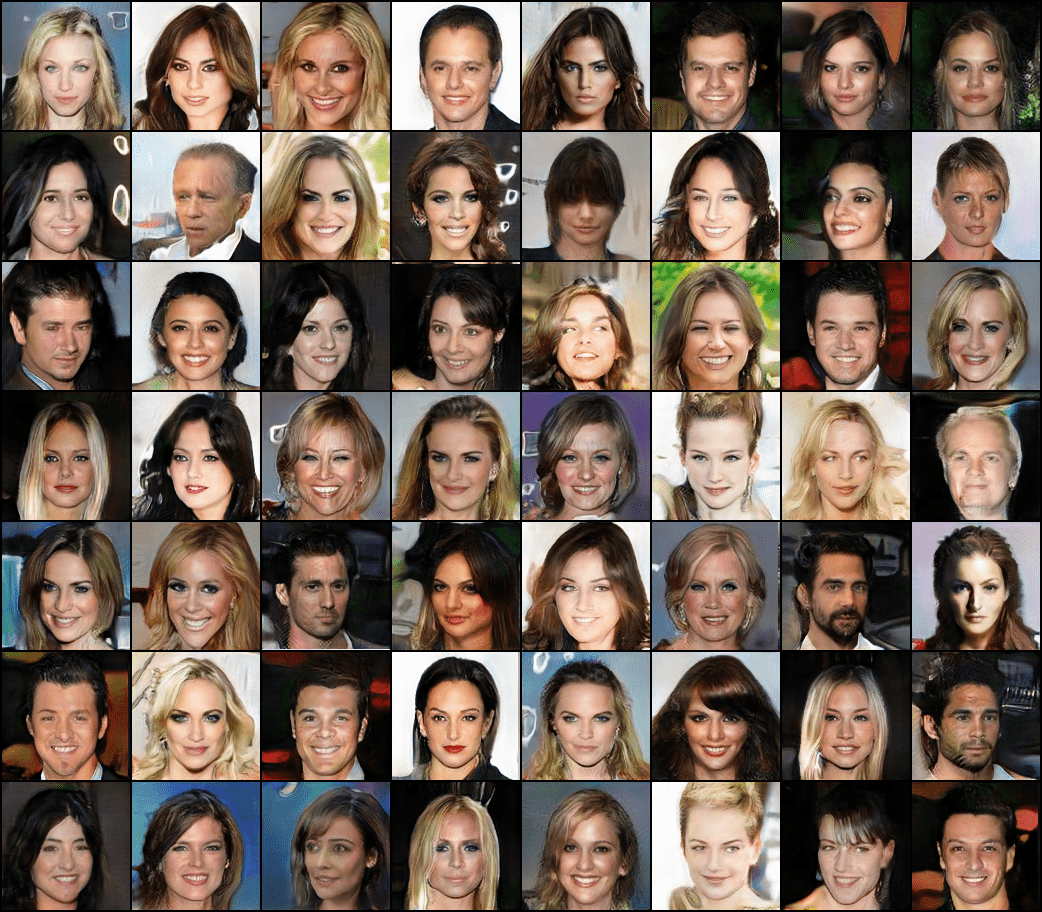}
    \caption{CelebA}
  \end{subfigure}
  \hfill
  \begin{subfigure}[t]{0.5\textwidth}
    \centering
    \includegraphics[width=6.8cm]{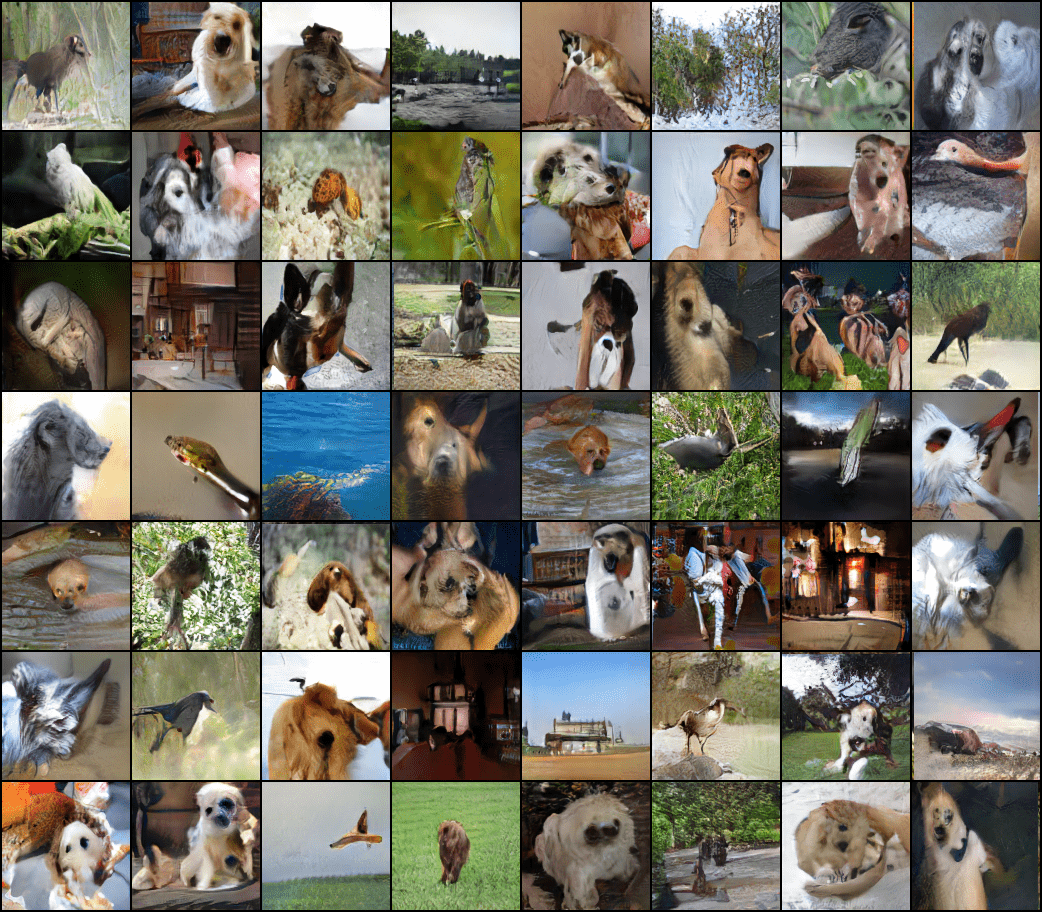}
    \caption{ImageNet128}
  \end{subfigure}
  \caption{Samples from $128 \times 128$ CelebA and ImageNet128 datasets}
  \label{fig:samples}
\end{figure}

\subsection{Quantitative analysis}
We further investigate the quality of the reconstructions with a quantitative assessment of the preservation of perceptual features in the input sample. For 
this evaluation task, we use the CelebA dataset where each image comes with a 
40 dimensional binary attributes vector. A VGG-16 classifier\citep{SimonyanZ14a} 
was trained on the CelebA training set to classify the individual attributes. This 
trained model is then used to classify the attributes of the reconstructions from
the validation set. We consider a reconstruction as being good if it preserves -- as measured by the trained classifier -- the attributes possessed by the original sample. 

We report a summary of the statistics of the classifier's accuracies 
in \autoref{table:celeba_accuracies_summary}.  
We do this for three different models, VAE,
ALI and HALI. An inspection of the table reveals that 
the proportion of attributes where HALI's reconstructions outperforms the
other models is clearly dominant.
Therefore, the encoder-decoder relationship of HALI better preserves the identifiable 
attributes compared to other models leveraging such relationships. Please refer to
\autoref{table:celeba_accuracies} in the appendix for the full table of attributes score.

\begin{table}[]
    \small
    \centering
    \begin{tabular}{l|c|c|c}
    \hline
                    & \textbf{Mean} & \textbf{Std} & \textbf{\# Best} \\
    \hline
    Data            & $77.13$ & $12.48$ & \ \\
    VAE             & $81.28$ & $10.50$ & $5$ \\
    ALI             & $84.60$ & $5.73$ & $3$ \\
    HALI $\bm{z}_1$ & $91.35$ & $5.62$ & $27$ \\
    HALI $\bm{z}_2$ & $86.28$ & $5.64$ & $3$ \\
    \hline
    \end{tabular}
    \caption{Summary of CelebA attributes classification from reconstructions for
        VAE, ALI and the two levels of HALI. The data row is the summary for
        the VGG classifier and the other scores have been normalized by it. 
        Mean and standard deviation are expressed as percentages. \# best represents
        the count of when a model has the best score on a single attribute.
        Note that it does not sum to 40 as there were ties.}
    \label{table:celeba_accuracies_summary}
\end{table}

\subsection{Perceptual Reconstructions}
In the same spirit as~\cite{larsen2015autoencoding}, 
we construct a metric by computing the Euclidean
distance between the input images and their various reconstructions in the
discriminator's feature space. More precisely, let $\cdot \mapsto \bar{D}(\cdot)$ be the embedding of the
input to the pen-ultimate layer of the discriminator.
We compute the \emph{discriminator embedded distance}
\begin{equation}
	\label{eq:disc_embbed}
  d_{c}(\bm{u}, \bm{v}) = \norm{\bar{D}(\bm{u}, \bm{\hat{u}_1}, \bm{\hat{u}_2}) - \bar{D}(\bm{v}, \bm{\hat{v}}_1, \bm{\hat{v}_2})}_{2},
\end{equation}
where $\cdot \mapsto \norm{\cdot}_2$ is the Euclidean norm. 
We then compute the average distances
$d_{c}(\bm{x}, \hat{\bm{x}}_1)$ and $d_{c}(\bm{x}, \hat{\bm{x}}_2)$ over the ImageNet 
validation set. \autoref{fig:subim1} shows that under $d_{c}$, the average
reconstruction errors for both $\hat{\bm{x}}_1$ and $\hat{\bm{x}}_2$
decrease steadily as training advances. Furthermore, the reconstruction error under $d_c$ of the reconstructions from the first level of the hierarchy are uniformly bounded by above by those of the second.
We note that while the VAEGAN model
of~\cite{larsen2015autoencoding} explicitly minimizes the perceptual
reconstruction error by adding this term to their loss function, HALI implicitly
minimizes it during adversarial training, as shown in \autoref{sec:high_recon}.

\begin{figure}[h]
\begin{subfigure}{0.49\textwidth}
  \includegraphics[width=0.99\linewidth, height=4cm]{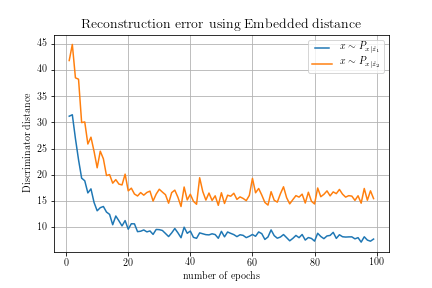}
  
\caption{}
\label{fig:subim1}
\end{subfigure}
\begin{subfigure}{0.49\textwidth}
\includegraphics[width=0.99\linewidth, height=4cm]{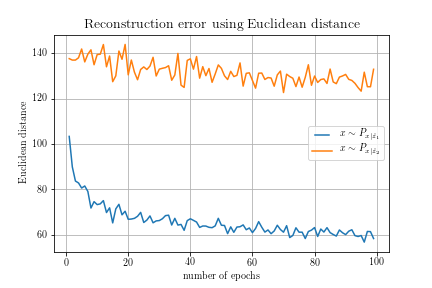}

\caption{}
\label{fig:imagenet}
\end{subfigure}
\caption{\small Comparison of average reconstruction error over the validation
    set for each level of reconstructions using the Euclidean (a) and
    discriminator embedded (b) distances. Using both distances, reconstructions
    errors for $\bm{x} \sim \M{\bm{x}}{\bm{z}_1}$ are uniformly below those for
    $\bm{x} \sim \M{\bm{x}}{\bm{z}_2}$. The reconstruction error using the
    Euclidean distance eventually stalls showing that the Euclidean metric
    poorly approximates the manifold of natural images.
}
\end{figure}

\section{Empirical Analysis II: Learned Representations}\label{subsec:exp_lr}
We now move on to assessing the quality of our learned representation through inpainting,
visualizing the hierarchy and innovation vectors. 

\subsection{Inpainting}
Inpainting is the task of reconstructing the missing or lost parts of an image. It is a challenging task since sufficient prior information is needed to meaningfully replace the 
missing parts of an image. 
While it is common to incorporate inpainting-specific training \cite{YehCLHD16,Perez,PathakKDDE16}, in our case we simply use the standard HALI adversarial loss during training and reconstruct incomplete images during inference time.

\begin{figure}[H]
  \begin{subfigure}[t]{0.33\textwidth}
    \centering
    \includegraphics[width=4.6cm,height=3cm]{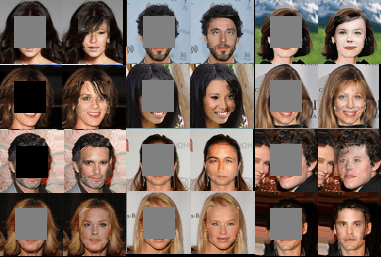}
  \end{subfigure}
  \hfill
  \begin{subfigure}[t]{0.33\textwidth}
    \centering
    \includegraphics[width=4.2cm,height=3cm]{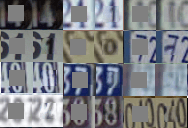}
  \end{subfigure}
  \hfill
  \begin{subfigure}[t]{0.24\textwidth}
    \centering
    \includegraphics[width=3.2cm,height=3cm]{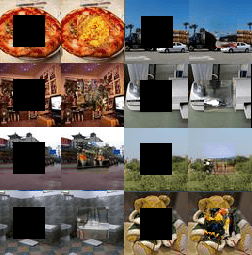}
  \end{subfigure}
  \caption{\small Inpainting on center cropped images on CelebA, SVHN and MS-COCO datasets
           (left to right).}
  \label{fig:inpaintings}
\end{figure}

We first predict the missing portions from the higher level reconstructions 
followed by iteratively using the lower level reconstructions that are pixel-wise closer to the original image. \autoref{fig:inpaintings} shows the inpaintings on 
center-cropped SVHN, CelebA and MS-COCO~\citep{lin2014microsoft} datasets without any 
blending post-processing or explicit supervision. The effectiveness of our model at this task is due the hierarchy -- we can extract semantically consistent reconstructions from the higher levels of the hierarchy, then leverage
pixel-wise reconstructions from the lower levels.

\begin{figure}[h]

\centering
\includegraphics[width=13.2cm,height=0.90cm]{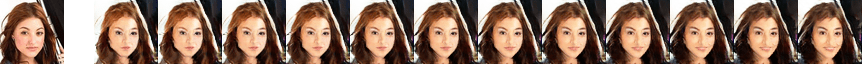}\\\vfill
\includegraphics[width=13.2cm,height=0.92cm]{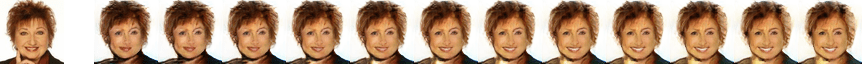}\\\vfill
\includegraphics[width=13.2cm,height=0.86cm]{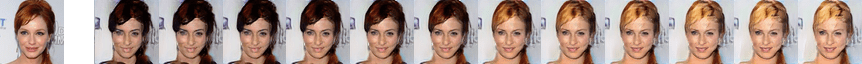}\\\vfill
\includegraphics[width=13.2cm,height=0.88cm]{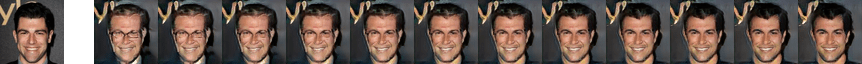}

\caption{\small Real CelebA faces (right) and their corresponding innovation tensor (IT) 
    updates (left). For instance, the third row in the figure features Christina Hendricks followed 
    by hair-color IT updates. Similarly, the first two rows depicts usage 
    of smile-IT and the 4th row glasses-plus-hair-color-IT.}
\label{fig:innovation}
\end{figure}

\subsection{Hierarchical latent representations}
To qualitatively show that higher levels of the hierarchy encode
increasingly abstract representation of the data,
we individually vary the latent variables and observe the effect.

The process is as follow: we sample a latent code from 
the prior distribution $\bm{z}_2$. We then multiply individual components of the
vector by scalars ranging from $-3$ to $3$. 
For $\bm{z}_1$, we fix $\bm{z}_2$ and multiply each feature map independently by scalars
ranging from $-3$ to $3$.
In all cases these modified latent vectors are then decoded back to input data space.
\autoref{fig:z_repr} (a) and (b) exhibit some of those decodings for $\bm{z}_2$,
while (c) and (d) do the same for the lower conditional $\bm{z}_1$. 
The last column contain the
decodings obtained from the originally sampled latent codes. We see that the representations
learned in the $\bm{z}_2$ conditional are responsible for high level
variations like gender, 
while $\bm{z}_1$ codes imply local/pixel-wise changes such as saturation or lip color.

\begin{figure}[H]
  \begin{subfigure}[t]{0.49\textwidth}
    \centering
    \includegraphics[width=\textwidth]{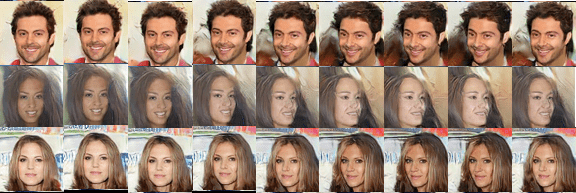}
    \caption{CelebA orientation variation}
  \end{subfigure}
  \hfill
  \begin{subfigure}[t]{0.49\textwidth}
    \centering
    \includegraphics[width=\textwidth]{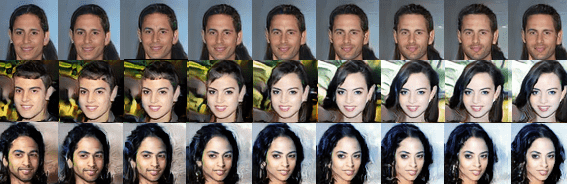}
    \caption{CelebA gender variation}
  \end{subfigure}
  \vfill
 \medskip
  \begin{subfigure}[t]{0.49\textwidth}
    
    \includegraphics[width=\textwidth,height=2.3cm]{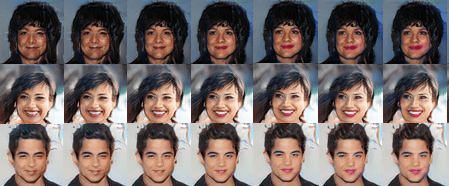}
    \caption{CelebA lipstick feature map variation}
  \end{subfigure}
   \hfill
  \begin{subfigure}[t]{0.49\textwidth}\centering
    \includegraphics[width=\textwidth,height=2.3cm]{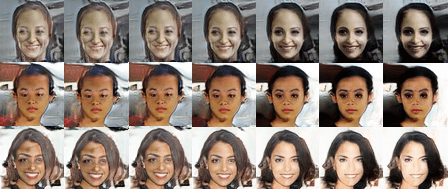}
 
    \caption{CelebA saturation feature map variation}
  \end{subfigure}
  \caption{\small (a) and (b) showcase $\bm{z}_2$ vector variation. We sample a set of $\bm{z}_2$ 
  vectors from the prior. We repeatedly replace a single relevant entry in each vector 
  by a scalar ranging from $-3$ to $3$ and decode. (c) and (d) follows the 
  same process using the $\bm{z}_1$ latent space.}
\label{fig:z_repr}
\end{figure}

\subsection{Latent semantic Innovation}
With HALI, we can exploit the jointly learned hierarchical inference
mechanism to modify actual data samples by manipulating their latent codes. We refer to these sorts of manipulations as latent semantic innovations. 

Consider a given instance from a dataset $\bm{x} \sim q(\bm{x})$. Encoding $\bm{x}$ yields 
$\hat{\bm{z}}_1$ and $\hat{\bm{z}}_2$. We modify $\hat{\bm{z}}_2$ by multiplying
a specific entry by a scalar $\alpha$. We denote the resulting vector by
$\hat{\bm{z}}_2^{\alpha}$. We decode the latter and get $\tilde{z}_1^{\alpha} \sim
\M{\bm{z}_1}{\bm{z}_2}$. We decode the unmodified encoding vector and get
$\tilde{\bm{z}}_1 \sim \M{\bm{z}_1}{\hat{\bm{z}}_2}$.
We then form the \emph{innovation tensor} $\eta^{\alpha} =
\tilde{\bm{z}}_1 - \tilde{\bm{z}}_1^{\alpha}$. Finally, we subtract the innovation
vector from the initial encoding, thus getting $\hat{\bm{z}}_1^{\alpha} =
\hat{\bm{z}}_1 - \eta^{\alpha}$, and sample $\tilde{\bm{x}}^{\alpha} \sim
\M{\bm{x}}{\hat{\bm{z}}^{\alpha}_1}$. 
This method provides explicit control and allows us to carry out 
these variations on real samples in a completely unsupervised way. The results are shown in 
\autoref{fig:innovation}. These were done on the CelebA validation set and were not used for training.

\section{Empirical Evaluations III: Learning Predictive Representations}\label{subsec:exp_tr}
We evaluate the usefulness of our learned representation for downstream tasks by quantifying the performance of HALI on attribute classification in CelebA and on a semi-supervised variant of the MNIST digit classification task.
\subsection{Unsupervised classification}
Following the protocol established by
~\cite{berg2013poof,liu2015deep},
we train 40 linear SVMs on HALI encoder representations (i.e. we utilize the inference network) on the CelebA validation set and subsequently measure performance on the test set.
As in ~\cite{berg2013poof,huang2016learning,kalayeh2017improving}, we report the \emph{balanced accuracy} in order to evaluate the attribute prediction performance. 
We emphasize that, for this experiment, the HALI encoder and decoder were trained in on entirely unsupervised data. Attribute labels were only used to train the linear SVM classifiers.

A summary of the results are reported in \autoref{table:att_acc_pred_summary}. 
HALI's unsupervised features 
surpass those of VAE and ALI, but more remarkably, they outperform the
best handcrafted features by a wide margin~\citep{zhang2014panda}. Furthermore, our approach outperforms a
number of supervised ~\citep{huang2016learning} and deeply supervised~\citep{liu2015deep}
features. \autoref{table:att_acc_pred} in the appendix arrays the results per attribute.


\begin{table}[]
    \small
    \centering
    \begin{tabular}{l|c|c|c}
    \hline
                & \textbf{Mean} & \textbf{Std} & \textbf{\# Best} \\
    \hline
    Triplet-kNN~\citep{schroff2015facenet} & $71.55$ & $12.61$ & $0$ \\
    PANDA~\citep{zhang2014panda}       & $76.95$ & $13.33$ & $0$ \\
    Anet~\citep{liu2015deep}        & $79.56$ & $12.17$ & $0$ \\
    LMLE-kNN~\citep{huang2016learning}    & $83.83$ & $12.33$ & $22$ \\
    VAE         & $73.30$ & $9.65$  & $0$ \\
    ALI         & $73.88$ & $10.16$ & $0$ \\
    HALI        & $83.75$ & $8.96$  & $15$ \\
    \hline
    \end{tabular}
    \caption{Summary of statistics for CelebA attributes mean per-class balanced accuracy
        given in percentage points. \# best represents
        the count of when a model has the best score on a single attribute.
        Note that it does not sum to 40 as there were ties.}
    \label{table:att_acc_pred_summary}
\end{table}

\subsection{Semi-supervised learning within HALI}
The HALI hierarchy can also be used in a more integrated semi-supervised setting, where the encoder also receives a training signal from the supervised objective. 
The currently most successful approach to semi-supervised in adversarially
trained generative models are built on the approach introduced
by~\cite{Salimans2016gan}. This formalism relies on exploiting the
discriminator's feature to differentiate between the individual classes present
in the labeled data as well as the generated samples.
Taking inspiration from~\citep{makhzani2015adversarial,makhzani2017pixelgan}, we adopt a different approach that leverages the Markovian hierarchical
inference network made available by HALI,

\begin{equation}
  \bm{x} \rightarrow \bm{z} \rightarrow \bm{y},
\end{equation}

Where $\bm{z} = enc(\bm{x} + \sigma \, \bm{\epsilon})$, with $\bm{\epsilon} \sim \mathcal{N}(0, \bm{I})$,
and $\bm{y}$ is a categorical random variable. In practice, we characterize the
conditional distribution of $\bm{y}$ given $\bm{z}$ by a softmax. The cost of the
generator is then augmented by a supervised cost. Let us write
$\mathcal{D}_{sup}$ as the set of pairs all labeled instance along with their
label, the supervised cost reads

\begin{equation}
  \mathcal{L}_{sup} = \frac{1}{\mid \mathcal{D}_{sup} \mid} \sum_{\substack{(\bm{y}, \bm{x}) \in \mathcal{D}_{sup}\\ \hat{\bm{y}} \sim q(\bm{y} \mid \bm{x})}} \sum_{k} \bm{y}_k \log(\hat{\bm{y}}_k).
\end{equation}

We showcased this approach on a semi-supervised variant of MNIST\citep{lecun1998mnist} digit classification task with 100 labeled examples evenly distributed across classes.

\autoref{table:semi} shows that HALI achieves a new state-of-the-art result for 
this setting. Note
that unlike ~\cite{dai2017good}, HALI uses no additional regularization.

\begin{table}[]
	\small
    \centering
    \begin{tabular}{l|l}
    \hline
                                                          & MNIST (\# errors) \\ \hline
    VAE (M1+M2)~\citep{kingma2014semi}                     & $233 \pm 14$        \\
    VAT~\citep{miyato2017virtual}                          & $136$             \\
    CatGAN~\citep{catgan2015}                              & $191 \pm 10$        \\
    Adversarial Autoencoder~\citep{makhzani2015adversarial}& $190 \pm 10$        \\
    PixelGAN~\citep{makhzani2017pixelgan}                  & $108 \pm 15$        \\
    ADGM~\citep{maaloe2016auxiliary}                       & $96 \pm 2$          \\
    Feature-Matching GAN~\citep{Salimans2016gan}           & $93 \pm 6.5$        \\
    Triple GAN~\citep{li2017triple}                        & $91 \pm 58$         \\
    GSSLTRABG~\citep{dai2017good}  & $79.5 \pm 9.8$      \\ \hline
    HALI (ours)                                           & $\textbf{73}$                \\ \hline
    \end{tabular}
    \caption{\small Comparison on semi-supervised learning with state-of-the-art methods on 
        MNIST with 100 labels instance per class. Only methods without data
        augmentation are included.}
    \label{table:semi}
\end{table}

\section{Conclusion and future work}
\label{conclusion}
In this paper, we introduced HALI, a novel adversarially trained generative model. HALI learns a hierarchy of latent variables with a simple Markovian structure in both the generator and inference networks. 
We have shown both theoretically and empirically the advantages gained by 
extending the ALI framework to a hierarchy. 

While there are many potential applications of HALI, one important future direction of research is to explore ways to render the training process more stable and straightforward. GANs are well-known to be challenging to train and the introduction of a hierarchy of latent variables only adds to this. 

\bibliography{iclr2018_conference}

\begin{thebibliography}{48}
\providecommand{\natexlab}[1]{#1}
\providecommand{\url}[1]{\texttt{#1}}
\expandafter\ifx\csname urlstyle\endcsname\relax
  \providecommand{\doi}[1]{doi: #1}\else
  \providecommand{\doi}{doi: \begingroup \urlstyle{rm}\Url}\fi

\bibitem[Berg \& Belhumeur(2013)Berg and Belhumeur]{berg2013poof}
Thomas Berg and Peter~N Belhumeur.
\newblock Poof: Part-based one-vs.-one features for fine-grained
  categorization, face verification, and attribute estimation.
\newblock In \emph{Proceedings of the IEEE Conference on Computer Vision and
  Pattern Recognition}, pp.\  955--962, 2013.

\bibitem[Chen et~al.(2016)Chen, Duan, Houthooft, Schulman, Sutskever, and
  Abbeel]{ChenDHSSA16}
Xi~Chen, Yan Duan, Rein Houthooft, John Schulman, Ilya Sutskever, and Pieter
  Abbeel.
\newblock Infogan: Interpretable representation learning by information
  maximizing generative adversarial nets.
\newblock \emph{CoRR}, abs/1606.03657, 2016.
\newblock URL \url{http://arxiv.org/abs/1606.03657}.

\bibitem[Cover \& Thomas(2012)Cover and Thomas]{cover2012elements}
Thomas~M Cover and Joy~A Thomas.
\newblock \emph{Elements of information theory}.
\newblock John Wiley \& Sons, 2012.

\bibitem[Dai et~al.(2017)Dai, Yang, Yang, Cohen, and
  Salakhutdinov]{dai2017good}
Zihang Dai, Zhilin Yang, Fan Yang, William~W Cohen, and Ruslan Salakhutdinov.
\newblock Good semi-supervised learning that requires a bad gan.
\newblock \emph{arXiv preprint arXiv:1705.09783}, 2017.

\bibitem[Denton et~al.(2015)Denton, Chintala, Fergus, et~al.]{denton2015deep}
Emily~L Denton, Soumith Chintala, Rob Fergus, et~al.
\newblock Deep generative image models using a laplacian pyramid of adversarial
  networks.
\newblock In \emph{Advances in Neural Information Processing Systems}, pp.\
  1486--1494, 2015.

\bibitem[Donahue et~al.(2016)Donahue, Kr{\"a}henb{\"u}hl, and
  Darrell]{donahue2016adversarial}
Jeff Donahue, Philipp Kr{\"a}henb{\"u}hl, and Trevor Darrell.
\newblock Adversarial feature learning.
\newblock \emph{arXiv preprint arXiv:1605.09782}, 2016.

\bibitem[Dumoulin et~al.(2016)Dumoulin, Belghazi, Poole, Lamb, Arjovsky,
  Mastropietro, and Courville]{dumoulin2016adversarially}
Vincent Dumoulin, Ishmael Belghazi, Ben Poole, Alex Lamb, Martin Arjovsky,
  Olivier Mastropietro, and Aaron Courville.
\newblock Adversarially learned inference.
\newblock \emph{arXiv preprint arXiv:1606.00704}, 2016.

\bibitem[Finn et~al.(2016)Finn, Goodfellow, and Levine]{FinnGL16}
Chelsea Finn, Ian~J. Goodfellow, and Sergey Levine.
\newblock Unsupervised learning for physical interaction through video
  prediction.
\newblock \emph{CoRR}, abs/1605.07157, 2016.
\newblock URL \url{http://arxiv.org/abs/1605.07157}.

\bibitem[Goodfellow et~al.(2014)Goodfellow, Pouget-Abadie, Mirza, Xu,
  Warde-Farley, Ozair, Courville, and Bengio]{goodfellow2014generative}
Ian Goodfellow, Jean Pouget-Abadie, Mehdi Mirza, Bing Xu, David Warde-Farley,
  Sherjil Ozair, Aaron Courville, and Yoshua Bengio.
\newblock Generative adversarial nets.
\newblock In \emph{Advances in Neural Information Processing Systems}, pp.\
  2672--2680, 2014.

\bibitem[Gulrajani et~al.(2016)Gulrajani, Kumar, Ahmed, Taiga, Visin,
  V{\'{a}}zquez, and Courville]{GulrajaniKATVVC16}
Ishaan Gulrajani, Kundan Kumar, Faruk Ahmed, Adrien~Ali Taiga, Francesco Visin,
  David V{\'{a}}zquez, and Aaron~C. Courville.
\newblock Pixelvae: {A} latent variable model for natural images.
\newblock \emph{CoRR}, abs/1611.05013, 2016.
\newblock URL \url{http://arxiv.org/abs/1611.05013}.

\bibitem[He et~al.(2015)He, Zhang, Ren, and Sun]{HeZRS15}
Kaiming He, Xiangyu Zhang, Shaoqing Ren, and Jian Sun.
\newblock Deep residual learning for image recognition.
\newblock \emph{CoRR}, abs/1512.03385, 2015.
\newblock URL \url{http://arxiv.org/abs/1512.03385}.

\bibitem[Huang et~al.(2016{\natexlab{a}})Huang, Li, Change~Loy, and
  Tang]{huang2016learning}
Chen Huang, Yining Li, Chen Change~Loy, and Xiaoou Tang.
\newblock Learning deep representation for imbalanced classification.
\newblock In \emph{Proceedings of the IEEE Conference on Computer Vision and
  Pattern Recognition}, pp.\  5375--5384, 2016{\natexlab{a}}.

\bibitem[Huang et~al.(2016{\natexlab{b}})Huang, Li, Poursaeed, Hopcroft, and
  Belongie]{HuangLPHB16}
Xun Huang, Yixuan Li, Omid Poursaeed, John~E. Hopcroft, and Serge~J. Belongie.
\newblock Stacked generative adversarial networks.
\newblock \emph{CoRR}, abs/1612.04357, 2016{\natexlab{b}}.
\newblock URL \url{http://arxiv.org/abs/1612.04357}.

\bibitem[Ioffe \& Szegedy(2015)Ioffe and Szegedy]{IoffeS15}
Sergey Ioffe and Christian Szegedy.
\newblock Batch normalization: Accelerating deep network training by reducing
  internal covariate shift.
\newblock \emph{CoRR}, abs/1502.03167, 2015.
\newblock URL \url{http://arxiv.org/abs/1502.03167}.

\bibitem[Kalayeh et~al.(2017)Kalayeh, Gong, and Shah]{kalayeh2017improving}
Mahdi~M Kalayeh, Boqing Gong, and Mubarak Shah.
\newblock Improving facial attribute prediction using semantic segmentation.
\newblock \emph{arXiv preprint arXiv:1704.08740}, 2017.

\bibitem[Karaletsos(2016)]{karaletsos2016adversarial}
Theofanis Karaletsos.
\newblock Adversarial message passing for graphical models.
\newblock \emph{NIPS workshop on Advances in Approximate Bayesian Inference},
  2016.

\bibitem[Kingma \& Welling(2013)Kingma and Welling]{kingma2013auto}
Diederik~P Kingma and Max Welling.
\newblock Auto-encoding variational bayes.
\newblock \emph{arXiv preprint arXiv:1312.6114}, 2013.

\bibitem[Kingma et~al.(2014)Kingma, Mohamed, Rezende, and
  Welling]{kingma2014semi}
Diederik~P Kingma, Shakir Mohamed, Danilo~Jimenez Rezende, and Max Welling.
\newblock Semi-supervised learning with deep generative models.
\newblock In \emph{Advances in Neural Information Processing Systems}, pp.\
  3581--3589, 2014.

\bibitem[Krizhevsky \& Hinton(2009)Krizhevsky and
  Hinton]{krizhevsky2009learning}
Alex Krizhevsky and Geoffrey Hinton.
\newblock Learning multiple layers of features from tiny images, 2009.

\bibitem[Lamb et~al.(2016)Lamb, Dumoulin, and
  Courville]{lamb2016discriminative}
Alex Lamb, Vincent Dumoulin, and Aaron Courville.
\newblock Discriminative regularization for generative models.
\newblock \emph{arXiv preprint arXiv:1602.03220}, 2016.

\bibitem[Larsen et~al.(2015)Larsen, S{\o}nderby, Larochelle, and
  Winther]{larsen2015autoencoding}
Anders Boesen~Lindbo Larsen, S{\o}ren~Kaae S{\o}nderby, Hugo Larochelle, and
  Ole Winther.
\newblock Autoencoding beyond pixels using a learned similarity metric.
\newblock \emph{arXiv preprint arXiv:1512.09300}, 2015.

\bibitem[LeCun et~al.(1998)LeCun, Cortes, and Burges]{lecun1998mnist}
Yann LeCun, Corinna Cortes, and Christopher~JC Burges.
\newblock The mnist database of handwritten digits, 1998.

\bibitem[Li et~al.(2017{\natexlab{a}})Li, Xu, Zhu, and Zhang]{li2017triple}
Chongxuan Li, Kun Xu, Jun Zhu, and Bo~Zhang.
\newblock Triple generative adversarial nets.
\newblock \emph{arXiv preprint arXiv:1703.02291}, 2017{\natexlab{a}}.

\bibitem[Li et~al.(2017{\natexlab{b}})Li, Liu, Chen, Pu, Chen, Henao, and
  Carin]{Li2017_ALICE}
Chunyuan Li, Hao Liu, Changyou Chen, Yunchen Pu, Liqun Chen, Ricardo Henao, and
  Lawrence Carin.
\newblock Alice: Towards understanding adversarial learning for joint
  distribution matching.
\newblock In \emph{Advances in Neural Information Processing Systems (NIPS)},
  2017{\natexlab{b}}.

\bibitem[Lin et~al.(2014)Lin, Maire, Belongie, Hays, Perona, Ramanan,
  Doll{\'a}r, and Zitnick]{lin2014microsoft}
Tsung-Yi Lin, Michael Maire, Serge Belongie, James Hays, Pietro Perona, Deva
  Ramanan, Piotr Doll{\'a}r, and C~Lawrence Zitnick.
\newblock Microsoft coco: Common objects in context.
\newblock In \emph{European conference on computer vision}, pp.\  740--755.
  Springer, 2014.

\bibitem[Liu et~al.(2015)Liu, Luo, Wang, and Tang]{liu2015deep}
Ziwei Liu, Ping Luo, Xiaogang Wang, and Xiaoou Tang.
\newblock Deep learning face attributes in the wild.
\newblock In \emph{Proceedings of the IEEE International Conference on Computer
  Vision}, pp.\  3730--3738, 2015.

\bibitem[Maal{\o}e et~al.(2016)Maal{\o}e, S{\o}nderby, S{\o}nderby, and
  Winther]{maaloe2016auxiliary}
Lars Maal{\o}e, Casper~Kaae S{\o}nderby, S{\o}ren~Kaae S{\o}nderby, and Ole
  Winther.
\newblock Auxiliary deep generative models.
\newblock \emph{arXiv preprint arXiv:1602.05473}, 2016.

\bibitem[Makhzani \& Frey(2017)Makhzani and Frey]{makhzani2017pixelgan}
Alireza Makhzani and Brendan Frey.
\newblock Pixelgan autoencoders.
\newblock \emph{arXiv preprint arXiv:1706.00531}, 2017.

\bibitem[Makhzani et~al.(2015)Makhzani, Shlens, Jaitly, and
  Goodfellow]{makhzani2015adversarial}
Alireza Makhzani, Jonathon Shlens, Navdeep Jaitly, and Ian Goodfellow.
\newblock Adversarial autoencoders.
\newblock \emph{arXiv preprint arXiv:1511.05644}, 2015.

\bibitem[Miyato et~al.(2017)Miyato, Maeda, Koyama, and
  Ishii]{miyato2017virtual}
Takeru Miyato, Shin-ichi Maeda, Masanori Koyama, and Shin Ishii.
\newblock Virtual adversarial training: a regularization method for supervised
  and semi-supervised learning.
\newblock \emph{arXiv preprint arXiv:1704.03976}, 2017.

\bibitem[Netzer et~al.(2011)Netzer, Wang, Coates, Bissacco, Wu, and
  Ng]{netzer2011reading}
Yuval Netzer, Tao Wang, Adam Coates, Alessandro Bissacco, Bo~Wu, and Andrew~Y
  Ng.
\newblock Reading digits in natural images with unsupervised feature learning.
\newblock In \emph{NIPS workshop on deep learning and unsupervised feature
  learning}, volume 2011, pp.\ ~4. Granada, Spain, 2011.

\bibitem[Nguyen et~al.(2016)Nguyen, Yosinski, Bengio, Dosovitskiy, and
  Clune]{NguyenYBDC16}
Anh Nguyen, Jason Yosinski, Yoshua Bengio, Alexey Dosovitskiy, and Jeff Clune.
\newblock Plug {\&} play generative networks: Conditional iterative generation
  of images in latent space.
\newblock \emph{CoRR}, abs/1612.00005, 2016.
\newblock URL \url{http://arxiv.org/abs/1612.00005}.

\bibitem[Oord et~al.(2016{\natexlab{a}})Oord, Kalchbrenner, Espeholt, Vinyals,
  Graves, et~al.]{van2016conditional}
Aaron van~den Oord, Nal Kalchbrenner, Lasse Espeholt, Oriol Vinyals, Alex
  Graves, et~al.
\newblock Conditional image generation with pixelcnn decoders.
\newblock In \emph{Advances in Neural Information Processing Systems}, pp.\
  4790--4798, 2016{\natexlab{a}}.

\bibitem[Oord et~al.(2016{\natexlab{b}})Oord, Kalchbrenner, and
  Kavukcuoglu]{oord2016pixel}
Aaron van~den Oord, Nal Kalchbrenner, and Koray Kavukcuoglu.
\newblock Pixel recurrent neural networks.
\newblock \emph{arXiv preprint arXiv:1601.06759}, 2016{\natexlab{b}}.

\bibitem[Pathak et~al.(2016)Pathak, Kr{\"{a}}henb{\"{u}}hl, Donahue, Darrell,
  and Efros]{PathakKDDE16}
Deepak Pathak, Philipp Kr{\"{a}}henb{\"{u}}hl, Jeff Donahue, Trevor Darrell,
  and Alexei~A. Efros.
\newblock Context encoders: Feature learning by inpainting.
\newblock \emph{CoRR}, abs/1604.07379, 2016.
\newblock URL \url{http://arxiv.org/abs/1604.07379}.

\bibitem[P{\'e}rez et~al.(2003)P{\'e}rez, Gangnet, and Blake]{Perez}
Patrick P{\'e}rez, Michel Gangnet, and Andrew Blake.
\newblock Poisson image editing.
\newblock \emph{ACM Trans. Graph.}, 22\penalty0 (3):\penalty0 313--318, July
  2003.
\newblock ISSN 0730-0301.
\newblock \doi{10.1145/882262.882269}.
\newblock URL \url{http://doi.acm.org/10.1145/882262.882269}.

\bibitem[Rosca et~al.(2017)Rosca, Lakshminarayanan, Warde-Farley, and
  Mohamed]{rosca2017variational}
Mihaela Rosca, Balaji Lakshminarayanan, David Warde-Farley, and Shakir Mohamed.
\newblock Variational approaches for auto-encoding generative adversarial
  networks.
\newblock \emph{arXiv preprint arXiv:1706.04987}, 2017.

\bibitem[Russakovsky et~al.(2015)Russakovsky, Deng, Su, Krause, Satheesh, Ma,
  Huang, Karpathy, Khosla, Bernstein, et~al.]{russakovsky2015imagenet}
Olga Russakovsky, Jia Deng, Hao Su, Jonathan Krause, Sanjeev Satheesh, Sean Ma,
  Zhiheng Huang, Andrej Karpathy, Aditya Khosla, Michael Bernstein, et~al.
\newblock Imagenet large scale visual recognition challenge.
\newblock \emph{International Journal of Computer Vision}, 115\penalty0
  (3):\penalty0 211--252, 2015.

\bibitem[Salimans \& Kingma(2016)Salimans and Kingma]{salimans2016weight}
Tim Salimans and Diederik~P Kingma.
\newblock Weight normalization: A simple reparameterization to accelerate
  training of deep neural networks.
\newblock In \emph{Advances in Neural Information Processing Systems}, pp.\
  901--901, 2016.

\bibitem[Salimans et~al.(2016)Salimans, Goodfellow, Zaremba, Cheung, Radford,
  and Chen]{Salimans2016gan}
Tim Salimans, Ian~J. Goodfellow, Wojciech Zaremba, Vicki Cheung, Alec Radford,
  and Xi~Chen.
\newblock Improved techniques for training gans.
\newblock \emph{arXiv preprint arXiv:1606.03498}, 2016.

\bibitem[Schroff et~al.(2015)Schroff, Kalenichenko, and
  Philbin]{schroff2015facenet}
Florian Schroff, Dmitry Kalenichenko, and James Philbin.
\newblock Facenet: A unified embedding for face recognition and
  clustering-1a\_089. pdf.
\newblock 2015.

\bibitem[Shrivastava et~al.(2016)Shrivastava, Pfister, Tuzel, Susskind, Wang,
  and Webb]{ShrivastavaPTSW16}
Ashish Shrivastava, Tomas Pfister, Oncel Tuzel, Josh Susskind, Wenda Wang, and
  Russ Webb.
\newblock Learning from simulated and unsupervised images through adversarial
  training.
\newblock \emph{CoRR}, abs/1612.07828, 2016.
\newblock URL \url{http://arxiv.org/abs/1612.07828}.

\bibitem[Simonyan \& Zisserman(2014)Simonyan and Zisserman]{SimonyanZ14a}
Karen Simonyan and Andrew Zisserman.
\newblock Very deep convolutional networks for large-scale image recognition.
\newblock \emph{CoRR}, abs/1409.1556, 2014.
\newblock URL \url{http://arxiv.org/abs/1409.1556}.

\bibitem[Springenberg(2015)]{catgan2015}
Jost~Tobias Springenberg.
\newblock Unsupervised and semi-supervised learning with categorical generative
  adversarial networks.
\newblock \emph{arXiv preprint arXiv:1511.06390}, 2015.

\bibitem[Ulyanov et~al.(2017)Ulyanov, Vedaldi, and Lempitsky]{Ulyanov2017}
Dmitry Ulyanov, Andrea Vedaldi, and Victor Lempitsky.
\newblock Adversarial generator-encoder networks.
\newblock \emph{arXiv preprint arXiv:1704.02304}, 2017.

\bibitem[Yeh et~al.(2016)Yeh, Chen, Lim, Hasegawa{-}Johnson, and Do]{YehCLHD16}
Raymond Yeh, Chen Chen, Teck{-}Yian Lim, Mark Hasegawa{-}Johnson, and Minh~N.
  Do.
\newblock Semantic image inpainting with perceptual and contextual losses.
\newblock \emph{CoRR}, abs/1607.07539, 2016.
\newblock URL \url{http://arxiv.org/abs/1607.07539}.

\bibitem[Zhang et~al.(2014)Zhang, Paluri, Ranzato, Darrell, and
  Bourdev]{zhang2014panda}
Ning Zhang, Manohar Paluri, Marc'Aurelio Ranzato, Trevor Darrell, and Lubomir
  Bourdev.
\newblock Panda: Pose aligned networks for deep attribute modeling.
\newblock In \emph{Proceedings of the IEEE Conference on Computer Vision and
  Pattern Recognition}, pp.\  1637--1644, 2014.

\bibitem[Zhu et~al.(2017)Zhu, Park, Isola, and Efros]{ZhuPIE17}
Jun{-}Yan Zhu, Taesung Park, Phillip Isola, and Alexei~A. Efros.
\newblock Unpaired image-to-image translation using cycle-consistent
  adversarial networks.
\newblock \emph{CoRR}, abs/1703.10593, 2017.
\newblock URL \url{http://arxiv.org/abs/1703.10593}.

\end{thebibliography}
\bibliographystyle{iclr2018_conference}
\clearpage
\appendix
\section{Appendix}

\subsection{Architecture Details}
\begin{table}[H]
\centering
\begin{tabular}{@{}rllllll@{}} 
Operation              & Kernel       & Strides      & Feature maps & BN/WN?      & Dropout & Nonlinearity \\ \midrule
$G_{z_1}(x)$ -- $3 \times 128 \times 128$ input                                                             \\
Convolution            & $3 \times 3$ & $1 \times 1$ & $32$         & $\times$  & 0.0     & Leaky ReLU \\
Convolution            & $3 \times 3$ & $2 \times 2$ & $64$         & BN  & 0.0     & Leaky ReLU \\
Resnet Block           & $3 \times 3$ & $1 \times 1$ & $64$        & BN  & 0.0     & Leaky ReLU \\
Resnet Block           & $3 \times 3$ & $1 \times 1$ & $64$        & BN  & 0.0     & Leaky ReLU \\
Convolution            & $3 \times 3$ & $2 \times 2$ & $128$        & BN & 0.0     & Leaky ReLU \\
Resnet Block            & $3 \times 3$ & $1 \times 1$ & $128$        & BN  & 0.0     & Leaky ReLU \\
Convolution            & $3 \times 3$ & $1 \times 1$ & $128$        & $\times$ & 0.0     & Leaky ReLU     \\
 Gaussian Layer \\
$G_{z_2}(z_1)$ -- $64 \times 32 \times 32$ input                                                              \\
Convolution            & $3 \times 3$ & $2 \times 2$ & $256$         & BN  & 0.0     & Leaky ReLU \\
Convolution            & $3 \times 3$ & $2 \times 2$ & $256$         & BN  & 0.0     & Leaky ReLU \\
Convolution           & $3 \times 3$ & $2 \times 2$ & $512$        & BN  & 0.0     & Leaky ReLU \\
Resnet Block            & $3 \times 3$ & $1 \times 1$ & $512$        & BN  & 0.0     & Leaky ReLU \\
Convolution             & $4 \times 4$ & $valid$ & $512$        & BN  & 0.0     & Leaky ReLU \\
Convolution            & $1 \times 1$ & $1 \times 1$ & $512$        & $\times$ & 0.0     & Linear     \\ \midrule
$G_{z_1}(z_2)$ -- $128 \times 1 \times 1$ input                                                              \\
Convolution            & $1 \times 1$ & $1 \times 1$ & $4 * 4 * 256$         & BN  & 0.0     & Leaky ReLU 		\\
Bilinear Upsampling \\
Resnet Block             & $3 \times 3$ & $1 \times 1$ & $256$         & BN & 0.0     & Leaky ReLU \\
Bilinear Upsampling \\
Convolution           & $3 \times 3$ & $1 \times 1$ & $256$        & BN  & 0.0     & Leaky ReLU \\
Convolution            & $3 \times 3$ & $1 \times 1$ & $128$        & BN  & 0.0     & Leaky ReLU \\
Bilinear Upsampling \\
Convolution             & $3 \times 3$ & $1 \times 1$ & $128$        & $\times$  & 0.0     & Leaky ReLU \\
 Gaussian Layer \\
$G_{x}(z_1)$ -- $64 \times 32 \times 32$ input                                                              \\
Convolution            & $1 \times 1$ & $1 \times 1$ & $64$         & BN  & 0.0     & Leaky ReLU 		\\

Resnet Block             & $3 \times 3$ & $1 \times 1$ & $64$         & BN  & 0.0     & Leaky ReLU \\
Bilinear Upsampling \\
Resnet Block             & $3 \times 3$ & $1 \times 1$ & $64$         & BN  & 0.0     & Leaky ReLU \\
Bilinear Upsampling \\
Convolution           & $3 \times 3$ & $1 \times 1$ & $64$        & BN  & 0.0     & Leaky ReLU \\
Convolution            & $3 \times 3$ & $1 \times 1$ & $32$        & BN  & 0.0     & Leaky ReLU \\

Convolution            & $1 \times 1$ & $1 \times 1$ & $3$        & $\times$  & 0.0     & Tanh \\ \midrule

$D(x)$ -- $3 \times 128 \times 128$ input                                                               \\
Convolution            & $3 \times 3$ & $1 \times 1$ & $32$         & WN & 0.2     & Leaky ReLU     \\
Convolution            & $3 \times 3$ & $2 \times 2$ & $64$         & WN & 0.5     & Leaky ReLU     \\
Convolution            & $3 \times 3$ & $2 \times 2$ & $64$        & WN & 0.5     & Leaky ReLU     \\
Convolution            & $3 \times 3$ & $1 \times 1$ & $64$        & WN & 0.5     & Leaky ReLU     \\
$D(x, z_1)$ -- $128 \times 32 \times 32$ input                                                           \\
\multicolumn{7}{@{}c@{}}{\em Concatenate $D(x)$ and $z_1$ along the channel axis}
	\\
Convolution            & $3 \times 3$ & $1 \times 1$ & $128$         & WN & 0.2     & Leaky ReLU     \\
Convolution            & $3 \times 3$ & $2 \times 2$ & $256$         & WN & 0.5     & Leaky ReLU     \\
Convolution            & $3 \times 3$ & $2 \times 2$ & $256$        & WN & 0.5     & Leaky ReLU     \\
Convolution            & $3 \times 3$ & $2 \times 2$ & $512$        & WN & 0.5     & Leaky ReLU     \\
Convolution            & $4 \times 4$ & $valid$ & $512$         & WN & 0.2     & Leaky ReLU     \\
$D(x, z_1,z_2)$ -- $512 \times 1 \times 1$ input                                                           \\
\multicolumn{7}{@{}c@{}}{\em Concatenate $D(x,z_1)$ and $z_2$ along the channel axis}
	\\

Convolution            & $1 \times 1$ & $1 \times 1$ & $1024$       & $\times$ & 0.5     & Leaky ReLU     \\
Convolution            & $1 \times 1$ & $1 \times 1$ & $1024$       & $\times$ & 0.5     & Leaky ReLU     \\
Convolution            & $1 \times 1$ & $1 \times 1$ & $1$          & $\times$ & 0.5     & Sigmoid    \\ \midrule
\end{tabular}
\vspace{0.2cm} 
\caption{\label{tab:imagenet_celebA_arch_description} Architecture detail for HALI(unsupervised) on the Imagenet 128 and CelebA 128 Datasets.}

\end{table}

\subsection{Proofs}
\begin{lemma}
\label{lem:joint_f_bound}
Let $f$ be a valid f-divergence generator.
Let $p$ and $q$ be joint distributions over a random vector $\bm{x}$. Let  $\bm{x}_A$ be any strict subset of $\bm{x}$ and $\bm{x}_{-A}$ its complement, then
\begin{equation}
\FD{p(\bm{x})}{q(\bm{x})} \geq \FD{p(\bm{x}_A)}{q(\bm{x}_A)}
\end{equation}
\end{lemma}
\begin{proof}
By definition, we have
\begin{equation*}
\FD{p(\bm{x})}{q(\bm{x})} = \EE_{\bm{x} \sim q(\bm{x})}[f(\frac{p(\bm{x})}{q(\bm{x})})] = \EE_{\bm{x}_A \sim q(\bm{x}_A)} \EE_{\bm{x}_{-A} \sim q(\bm{x}_{-A} \mid \bm{x}_A)}[f(\frac{p(\bm{x}_A) \, p(\bm{x}_{-A} \mid \bm{x}_A)}{q(\bm{x}_A) \, 	q(\bm{x}_{-A} \mid \bm{x}_A)})]
\end{equation*}
Using that $f$ is convex, Jensen's inequality yields
\begin{equation*}
\FD{p(\bm{x})}{q(\bm{x})} \geq \EE_{\bm{x}_A \sim q(\bm{x}_A)} [ f(\EE_{\bm{x}_{-A} \sim q(\bm{x}_{-A} \mid \bm{x}_A)}\frac{p(\bm{x}_A) \, p(\bm{x}_{-A} \mid \bm{x}_A)}{q(\bm{x}_A) \, 	q(\bm{x}_{-A} \mid \bm{x}_A)})]
\end{equation*}
Simplifying the inner expectation on the right hand side, we conclude that
\begin{equation*}
\FD{p(\bm{x})}{q(\bm{x})} \geq \EE_{\bm{x}_A \sim q(\bm{x}_A)} [f( \frac{p(\bm{x}_A) }{q(\bm{x}_A)})]
\end{equation*}
\end{proof}

  \begin{lemma}[Kullback-Leibler's upper bound by Jensen-Shannon]
    \label{lem:Jupperbound}
      Assume that $p$ and $q$ are two probability distribution absolutely
      continuous with respect to each other. Moreover, assume that $q$ is
      bounded away from zero. Then, there exist a positive scalar $K$ such that
      \begin{equation}
      \KL{p}{q} \leq K \JS{p}{q}.
    \end{equation}
  \end{lemma}
  \begin{proof}
    We start by bounding the Kullblack-Leibler divergence by the
    $\chi^2$-distance. We have
    \begin{equation}
      \KL{p}{q} \leq \int \log(\frac{p(\bm{x})}{q(\bm{x})}) \, dx \leq \log(1 + \XX{p}{q}) \leq \XX{p}{q}
    \end{equation}
    The first inequality follows by Jensen's inequality. The third inequality follows by the Taylor expansion.
    Recall that both the $\chi^2$-distance and the Jensen-Shanon divergences are
    f-divergences with generators given by $f_{\chi^2}(t) = (t - 1)^2$ and $f_{JS}(t)
    = u \log(\frac{2t}{t + 1}) + \log(\frac{2t}{t + 1})$, respectively.
    We form the function $t \mapsto h(t) = \frac{f_{\chi^2}(t)}{f_{JS}(t)}$. $h$
    is strictly increasing on $[0, \infty)$. Since we are assuming $q$ to be bounded away from zero,
    we know that there is a constant $c_1$ such that $q(\bm{x}) > c_1$ for all
    $\bm{x}$. Subsequently for all $\bm{x}$, we have that
    $\frac{p(x)}{q(x)} \leq c_2 := \max_{x} \frac{p(x)}{c_1}$.
    Thus, for all $x$ we have $h(\frac{p(x)}{q(x)}) \leq K:= h(c_2)$ and hence$
    f_{\chi^2}(\frac{p}{q}) \leq K \, f_{JS}(\frac{p(x)}{q(x)})$. Intergrating with respect to $q$, we conclude
    $$\KL{p}{q} \leq \XX{p}{q} \leq K \, \JS{p}{q}$$
  \end{proof}
  \begin{proposition}
    Assuming $q(\bm{x}, \bm{z}_l) $ and $p(\bm{x}, \bm{z}_l)$ are positive for any $l \in \{1, \dots, L\}$. We have
      \begin{equation}
        \EE_{\bm{x} \sim q(\bm{x})}[\mathcal{L}^l(\bm{x})] - H(\bm{x} \mid \bm{z}_l) \leq K \JS{p(\bm{x}, \bm{z}_1, \dots, \bm{z}_L)}{q (\bm{x}, \bm{z}_1, \dots, \bm{z}_L)})
      \end{equation}
      Where $H(\bm{x} \mid \bm{z}_l)$ is computed under the encoder's distribution $q(\bm{x}, \bm{z}_l)$
    \end{proposition}
    \begin{proof}
        By elementary manipulations we have.
        \begin{equation*}
          \EE_{\bm{x} \sim q(\bm{x})}[\LL^l(\bm{x})] = \KL{p(\bm{x}, \bm{z}_l)}{q(\bm{x}, \bm{z}_l)} + H(\bm{x}_l \mid \bm{z}_l) - \KL{p(\bm{z}_l)}{q(\bm{z}_l)}
        \end{equation*}
        Where the conditional entropy  $H(\bm{x}_l \mid \bm{z}_l)$ is computed
        $q(\bm{x}, \bm{z}_l)$. By the non-negativity of the KL-divergence we obtain
        \begin{equation*}
          \EE_{\bm{x} \sim q(\bm{x})}[\LL^l(\bm{x})] \leq \KL{p(\bm{x}, \bm{z}_l)}{q(\bm{x}, \bm{z}_l)} + H(\bm{x}_l \mid \bm{z}_l)
        \end{equation*}
        Using lemma~\ref{lem:Jupperbound}, we have

        \begin{equation*}
          \EE_{\bm{x} \sim p_{d}(\bm{x})}[\LL^l(\bm{x})] - H(\bm{x} \mid \bm{z}_l) \leq K \JS{p(\bm{x}, \bm{z}_l)}{q (\bm{x}, \bm{z}_l)})
          \end{equation*}
          The Jensen-Shanon divergence being f-divergence, using Lemma~\ref{lem:joint_f_bound}, we conclude
         \begin{equation*}
          \EE_{\bm{x} \sim p_{d}(\bm{x})}[\LL^l(\bm{x})] - H(\bm{x} \mid \bm{z}_l) \leq K \JS{p(\bm{x}, \bm{z}_1, \dots, \bm{z}_L)}{q (\bm{x}, \bm{z}_1, \dots, \bm{z}_L)}).
         \end{equation*} 
        \end{proof}
        
   \begin{proposition}
     \label{prop:rec_to_mi}
     For any given latent variable $\bm{z}_{l}$, the reconstruction likelihood
     $\EE_{\bm{x} \sim q_{\bm{x}}}[\EE_{\bm{z} \sim \M{\bm{z}_l}{\bm{x}}}[-\log p(\bm{x}
     \mid \bm{z}_{l})]]$ is an upper bound on $H(\bm{x} \mid \bm{z}_l)$.
   \end{proposition}

    \begin{proof}
      By the non-negativity of the Kullback-Leibler divergence, we have that
      $$\KL{q(\bm{x} \mid \bm{z}_l)}{p(\bm{x} \mid \bm{z}_l)} \geq 0 \implies
      H[q(\bm{x} \mid \bm{z}_l)] \leq \EE_{
        \bm{z}_{l} \sim \M{\bm{z}_l}{\bm{x}}}[-\log(p(\bm{x} \mid \bm{z}_l))] $$.
      Integrating over the marginal and applying Fubini's theorem yields
      $$H(\bm{x} \mid \bm{z}_{l}) \leq \EE_{\bm{x} \sim p(x)}\EE_{\bm{z}_l \sim
        \M{\bm{z}_l}{\bm{x}}}[-\log(p(\bm{x} \mid \bm{z}_l))],$$
      where the conditional entropy $H(\bm{x} \mid \bm{z}_l)$ is computed under the
      encoder distribution.
    \end{proof}

\begin{figure}
\begin{minipage}{\textwidth}
  \begin{subfigure}[t]{0.49\textwidth}
    \centering
    \includegraphics[height=5cm,width=5cm]{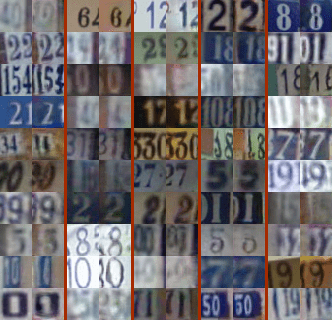}
    \caption{SVHN from $\bm{z}_1$}
  \end{subfigure}
  \begin{subfigure}[t]{0.49\textwidth}
    \centering
    \includegraphics[height=5cm,width=5cm]{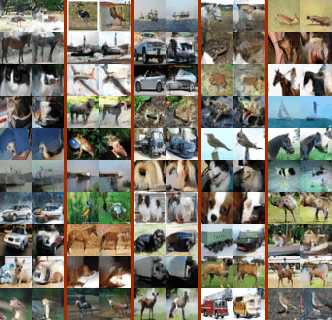}
    \caption{CIFAR10 from $\bm{z}_1$}
  \end{subfigure}
\end{minipage}%

\medskip
\begin{minipage}{\textwidth}
  \begin{subfigure}[t]{0.49\textwidth}
    \centering
    \includegraphics[height=5cm,width=5cm]{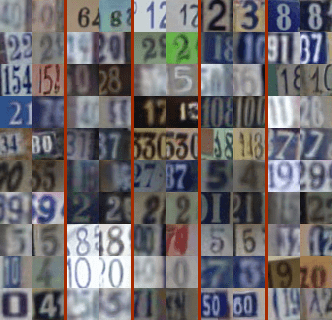}
    \caption{SVHN from $\bm{z}_2$}
  \end{subfigure}
  \begin{subfigure}[t]{0.49\textwidth}
    \centering
    \includegraphics[height=5cm,width=5cm]{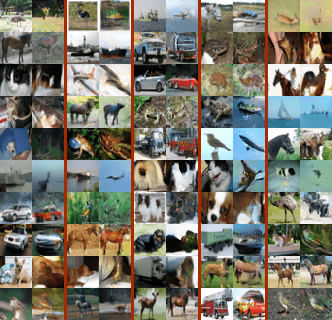}
    \caption{CIFAR10 from $\bm{z}_2$}
  \end{subfigure}
\end{minipage}%
  \caption{Reconstructions for SVHN and CIFAR10 from $\bm{z}_1$ and reconstructions from $\bm{z}_2$. 
    Odd columns corresponds to examples from the validation set while even columns are 
    the model's reconstructions} \label{fig:recons_appendix}
\end{figure}

\newpage

\begin{table}
  \small
  \centering
\newcommand*\rot{\rotatebox{80}}

\scalebox{0.68}{
\begin{tabular}{@{} cl*{20}c @{}}
    \toprule
    & & \rot{Sideburns} &         
    \rot{Black Hair} &
    \rot{Wavy Hair} &
    \rot{Young} &
    \rot{Makeup} &
    \rot{Blond} &
    \rot{Attractive} &
    \rot{withShadow} &
    \rot{withNecktie} &
    \rot{Blurry} &
    \rot{DoubleChin} &
    \rot{BrownHair} &
    \rot{Mouth Open} &
    \rot{Goatee} &
    \rot{Bald} &
    \rot{PointyNose} &
    \rot{Gray Hair} &
    \rot{Pale Skin} &
    \rot{ArchedBrows} &
    \rot{With Hat} \\

    & Data & 86 & 81 & 69 & 91 & 90 & 88 & 81 & 80 & 80 & 60 & 70 & 74 & 92 & 88 & 77 & 60 & 84 & 72 & 69 & 86 \\ 
    & VAE & 80 & \textbf{96} & 60 & \textbf{98} & 79 & 74 & 83 & 74 & \textbf{93} & \textbf{96} & 73 & 77 & 83 & 77 & 89 & 86 & 69 & 86 & 80 & 84 \\
    & ALI & 72 & 90 & 83 & 94 & 88 & 81 & 91 & 77 & 83 & 83 & \textbf{89} & 76 & 77 & 72 & 83 & 87 & 84 & 79 & 92 & 92 \\
    & HALI (z1) & \textbf{92} & 93 & \textbf{93} & 98 & \textbf{91} & \textbf{95} & \textbf{95} & \textbf{92} & 89 & 92 & 81 & \textbf{95} & \textbf{92} & \textbf{89} & \textbf{94} & \textbf{98} & \textbf{88} & 90 & \textbf{95} & \textbf{99} \\
    & HALI (z2) & 80 & 78 & 91 & 96 & 89 & 94 & 92 & 78 & 90 & 83 & 88 & 80 & 83 & 78 & 88 & 89 & 87 & \textbf{91} & 87 & 84 \\

    \midrule

    & & \rot{Balding} &
    \rot{StraightHair} &
    \rot{Big Nose} &
    \rot{Rosy Cheeks}  &
    \rot{Oval Face}  &
    \rot{Bangs}   &
    \rot{Male}   &
    \rot{Mustache}   &
    \rot{HighCheeks}  &
    \rot{No Beard}   &
    \rot{Eyeglasses}  &
    \rot{BaggyEyes}  &
    \rot{WithNecklace}  &
    \rot{WithLipstick}  &
    \rot{Big Lips}  &
    \rot{NarrowEyes}  &
    \rot{Chubby}  &
    \rot{Smiling}  &
    \rot{BushyBrows}  &
    \rot{WithEarrings} \\

    & Data & 72 & 60 & 69 & 71 & 56 & 89 & 97 & 80 & 85 & 96 & 94 & 69 & 54 & 94 & 54 & 58 & 71 & 92 & 72 & 74 \\
    & VAE  & 89 & \textbf{94} & 69 & 67 & 83 & 78 & 93 & 82 & 91 & 98 & 73 & 65 & 92 & 86 & 82 & 78 & 73 & 95 & 64 & 62 \\
    & ALI  & 86 & 87 & 89 & \textbf{82} & 89 & 76 & 90 & 77 & 86 & 92 & 81 & 87 & 93 & 88 & 83 & 85 & \textbf{90} & 85 & 82 & 83 \\
    & HALI (z1) & \textbf{96} & 91 & \textbf{90} & 71 & \textbf{91} & \textbf{94} & \textbf{97} & \textbf{84} & \textbf{96} & 98 & \textbf{90} & \textbf{94} & 93 & \textbf{94} & 83 & \textbf{88} & 82 & \textbf{97} & \textbf{91} & 83 \\
    & HALI (z2) & 88 & 88 & 89 & 75 & 90 & 84 & 90 & 80 & 89 & 95 & 75 & 86 & \textbf{99} & 90 & \textbf{88} & 87 & 87 & 89 & 85 & 81 \\

\bottomrule
\end{tabular}
}
  \caption{\small CelebA attributes accuracies of reconstructions by different models. 
  		   The data row displays the raw average of
  		   positive accuracies, predicting true 1, and negative accuracies, 
           predicting a true 0 by our
           VGG classifier. Other rows show the same average accuracies where
           each individual accuracy is normalized  by its corresponding data score. 
           The numbers are all percentages.}
  \label{table:celeba_accuracies}
\end{table}

\begin{table}
  \small
  \centering
  \newcommand*\rot{\rotatebox{80}}

  \scalebox{0.62}{
    \begin{tabular}{@{} cl*{20}c @{}}
      \toprule
      & & \rot{  5 o Clock Shadow   } & 
                                        \rot{  Arched Eyebrows      } & 
                                                                        \rot{  Attractive            } & 
                                                                                                         \rot{  Bags Under Eyes     } & 
                                                                                                                                        \rot{  Bald                  } & 
                                                                                                                                                                         \rot{  Bangs                 } & 
                                                                                                                                                                                                          \rot{  Big Lips             } & 
                                                                                                                                                                                                                                          \rot{  Big Nose             } & 
                                                                                                                                                                                                                                                                          \rot{  Black Hair           } & 
                                                                                                                                                                                                                                                                                                          \rot{  Blond Hair           } & 
                                                                                                                                                                                                                                                                                                                                          \rot{  Blurry                } & 
                                                                                                                                                                                                                                                                                                                                                                           \rot{  Brown Hair           } & 
                                                                                                                                                                                                                                                                                                                                                                                                           \rot{  Bushy Eyebrows       } & 
                                                                                                                                                                                                                                                                                                                                                                                                                                           \rot{  Chubby                } & 
                                                                                                                                                                                                                                                                                                                                                                                                                                                                            \rot{  Double Chin          } & 
                                                                                                                                                                                                                                                                                                                                                                                                                                                                                                            \rot{  Eyeglasses            } & 
                                                                                                                                                                                                                                                                                                                                                                                                                                                                                                                                             \rot{  Goatee                } & 
                                                                                                                                                                                                                                                                                                                                                                                                                                                                                                                                                                              \rot{  Gray Hair            } & 
                                                                                                                                                                                                                                                                                                                                                                                                                                                                                                                                                                                                              \rot{  Heavy Makeup         } & 
                                                                                                                                                                                                                                                                                                                                                                                                                                                                                                                                                                                                                                              \rot{  High Cheekbones      } \\

      & Triplet-kNN &  66          & 73          & 83          & 63          & 75          & 81          & 55          & 68          & 82          & 81          & 43          & 76          & 68          & 64          & 60          & 82          & 73          & 72          & 88          & 86           \\ 
      & PANDA &  76    & 77    & 85    & 67    & 74    & 92    & 56    & 72    & 84    & 91    & 50    & 85    & 74    & 65    & 64    & 88    & 84    & 79    & 95    & 89     \\ 
      & Anet       &  81          & 76          & 87          & 70          & 73          & 90          & 57          & 78          & 90          & 90          & 56          & 83          & \textbf{82} & 70          & 68          & 95          & 86          & 85          & 96          & 89           \\ 
      & LMLE-kNN    &  82          & \textbf{79} & \textbf{88} & 73          & 90          & \textbf{98} & 60          & \textbf{80} & \textbf{92} & \textbf{99} & 59          & \textbf{87} & \textbf{82} & 79          & 74          & \textbf{98} & \textbf{95} & 91          & \textbf{98} & \textbf{92}  \\ 
      & VAE &  78  & 65  & 62  & 68  & 87  & 86  & 58  & 67  & 75  & 83  & 64  & 62  & 72  & 77  & 80  & 81  & 80  & 88  & 75  & 75   \\ 
      & ALI &  78  & 70  & 69  & 68  & 89  & 87  & 57  & 69  & 75  & 88  & 65  & 64  & 71  & 78  & 78  & 85  & 79  & 89  & 79  & 64   \\ 
      & HALI(Unsup) &  \textbf{86}& 77         & 80         & \textbf{78}& \textbf{94}& 93         & \textbf{62}& 74         & 85         & 92         & 78         & 77         & 82         & \textbf{85}& \textbf{86}& 96         & 92         & \textbf{93}& 89         & 85          \\

      \midrule

      & & \rot{  Male                  } & 
                                           \rot{  Mouth Slightly Open } & 
                                                                          \rot{  Mustache              } & 
                                                                                                           \rot{  Narrow Eyes          } & 
                                                                                                                                           \rot{  No Beard             } & 
                                                                                                                                                                           \rot{  Oval Face            } & 
                                                                                                                                                                                                           \rot{  Pale Skin            } & 
                                                                                                                                                                                                                                           \rot{  Pointy Nose          } & 
                                                                                                                                                                                                                                                                           \rot{  Receding Hairline    } & 
                                                                                                                                                                                                                                                                                                           \rot{  Rosy Cheeks          } & 
                                                                                                                                                                                                                                                                                                                                           \rot{  Sideburns             } & 
                                                                                                                                                                                                                                                                                                                                                                            \rot{  Smiling               } & 
                                                                                                                                                                                                                                                                                                                                                                                                             \rot{  Straight Hair        } & 
                                                                                                                                                                                                                                                                                                                                                                                                                                             \rot{  Wavy Hair            } & 
                                                                                                                                                                                                                                                                                                                                                                                                                                                                             \rot{  Wearing Earrings     } & 
                                                                                                                                                                                                                                                                                                                                                                                                                                                                                                             \rot{  Wearing Hat          } & 
                                                                                                                                                                                                                                                                                                                                                                                                                                                                                                                                             \rot{  Wearing Lipstick     } & 
                                                                                                                                                                                                                                                                                                                                                                                                                                                                                                                                                                             \rot{  Wearing Necklace     } & 
                                                                                                                                                                                                                                                                                                                                                                                                                                                                                                                                                                                                             \rot{  Wearing Necktie      } & 
                                                                                                                                                                                                                                                                                                                                                                                                                                                                                                                                                                                                                                             \rot{  Young                 } \\

      & Triplet-kNN &  91          & 92          & 57          & 47          & 82          & 61          & 63          & 61          & 60          & 64          & 71          & 92          & 63          & 77          & 69          & 84          & 91          & 50          & 73          & 75           \\ 
      & PANDA &  99    & 93    & 63    & 51    & 87    & 66    & 69    & 67    & 67    & 68    & 81    & 98    & 66    & 78    & 77    & 90    & 97    & 51    & 85    & 78     \\ 
      & Anet        &  99          & 96          & 61          & 57          & 93          & 67          & 77          & 69          & 70          & 76          & 79          & 97          & 69          & 81          & 83          & 90          & 95          & 59          & 79          & 84           \\ 
      & LMLE-kNN    &  \textbf{99} & \textbf{96} & 73          & 59          & \textbf{96} & \textbf{68} & 80          & \textbf{72} & 76          & 78          & 88          & \textbf{99} & \textbf{73} & \textbf{83} & \textbf{83} & \textbf{99} & \textbf{99} & 59          & \textbf{90} & \textbf{87}  \\ 
      & VAE &  78  & 67  & 81  & 60  & 79  & 51  & 86  & 59  & 79  & 79  & 79  & 81  & 55  & 69  & 65  & 84  & 78  & 67  & 83  & 69   \\ 
      & ALI &  83  & 52  & 82  & 62  & 79  & 54  & 85  & 61  & 78  & 80  & 77  & 68  & 60  & 72  & 67  & 91  & 82  & 67  & 82  & 71   \\ 
      & HALI(Unsup) &  96         & 88         & \textbf{90}& \textbf{72}& 90         & 65         & \textbf{89}& 69         & \textbf{84}& \textbf{89}& \textbf{91}& 91         & 70         & 77         & 78         & 95         & 92         & \textbf{71}& 89         & 80          \\
      
      \bottomrule
    \end{tabular}
  }

  \caption{\small Mean per-class balanced accuracy in percentage points of each of the
    40 face attributes on CelebA.}
   \label{table:att_acc_pred}
\end{table}

\end{document}